\documentclass[letterpaper]{article} 
\usepackage{aaai24}  
\usepackage{times}  
\usepackage{helvet}  
\usepackage{courier}  
\usepackage[hyphens]{url}  
\usepackage{graphicx} 
\usepackage{natbib}  
\usepackage{caption} 
\usepackage{algorithm}
\usepackage{algorithmic}

\usepackage{newfloat}
\usepackage{listings}
\DeclareCaptionStyle{ruled}{labelfont=normalfont,labelsep=colon,strut=off} 
\floatstyle{ruled}
\newfloat{listing}{tb}{lst}{}
\floatname{listing}{Listing}
\title{MAPTree: Beating ``Optimal'' Decision Trees with Bayesian Decision Trees}
\author {
    Colin Sullivan\equalcontrib \textsuperscript{\rm 1},
    Mo Tiwari\equalcontrib \textsuperscript{\rm 1},
    Sebastian Thrun\textsuperscript{\rm 1}
}
\affiliations {
    \textsuperscript{\rm 1}Department of Computer Science, Stanford University\\
    colins26@stanford.edu, motiwari@stanford.edu, thrun@stanford.edu
}

\usepackage{amsmath}
\usepackage{amssymb}
\usepackage{mathtools}
\usepackage{amsthm}
\usepackage{bm}
\usepackage{xcolor}
\usepackage{tikz}
\usetikzlibrary{shapes.geometric, arrows}
\usepackage{subcaption}

\begin{document}

\theoremstyle{plain}
\newtheorem{theorem}{Theorem}
\newtheorem{proposition}[theorem]{Proposition}
\newtheorem{lemma}[theorem]{Lemma}
\newtheorem{corollary}[theorem]{Corollary}
\theoremstyle{definition}
\newtheorem{definition}[theorem]{Definition}
\newtheorem{assumption}[theorem]{Assumption}
\theoremstyle{remark}
\newtheorem{remark}[theorem]{Remark}
\newcommand\maybegeq{\stackrel{?}{\geq}}

\newcommand{\algnamenospace}{MAPTree}
\newcommand{\algname}{MAPTree }

\maketitle

\begin{abstract}

Decision trees remain one of the most popular machine learning models today, largely due to their out-of-the-box performance and interpretability.
In this work, we present a Bayesian approach to decision tree induction via maximum a posteriori inference of a posterior distribution over trees.
We first demonstrate a connection between maximum a posteriori inference of decision trees and AND/OR search.
Using this connection, we propose an AND/OR search algorithm, dubbed \algnamenospace, which is able to recover the maximum a posteriori tree.
Lastly, we demonstrate the empirical performance of the maximum a posteriori tree both on synthetic data and in real world settings.
On 16 real world datasets, \algname either outperforms baselines or demonstrates comparable performance but with much smaller trees.
On a synthetic dataset, \algname also demonstrates greater robustness to noise and better generalization than existing approaches.
Finally, \algname recovers the maxiumum a posteriori tree faster than existing sampling approaches and, in contrast with those algorithms, is able to provide a certificate of optimality.
The code for our experiments is available at \texttt{https://github.com/ThrunGroup/maptree}.

\end{abstract}

\section{Introduction}
\label{sec:introduction}


Decision trees are amongst the most widely used machine learning models today due to their empirical performance, generality, and interpretability.
A decision tree is a binary tree in which each internal node corresponds to an \texttt{if/then/else} comparison on a feature value; a label for a datapoint is produced by determining the corresponding leaf node into which it falls.
The predicted label is usually the majority vote (respectively, mean) of the label of training datapoints at the leaf node in classification (respectively, regression).

Despite recent advances in neural networks, decision trees remain a popular choice amongst machine learning practitioners.
Decision trees form the backbone of more complex ensemble models such as Random Forest \cite{breimanRandomForests2001} and XGBoost \cite{chenXGBoostScalableTree2016}, which have been the leading models in many machine learning competitions and often outperform neural networks on tabular data \cite{grinsztajnWhyTreeBasedModels2022}.
Decision trees naturally work with complex data where the features can be of mixed data types, e.g., binary, categorical, or continuous.
Furthermore, decision trees are highly interpretable and the prediction-generating process can be inspected, which can be a necessity in domains such as law and healthcare.
Furthermore, inference in decision trees is highly efficient as it relies only on efficient feature value comparisons.
Given decision trees' popularity, an improvement upon existing decision tree approaches would have widespread impact.



\textbf{Contributions:} In this work, we:
\begin{itemize}
    \item Formalize a connection between maximum a posteriori inference of Bayesian Classification and Regression Trees (BCART) and AND/OR search problems,
    \item Propose an algorithm, dubbed \algnamenospace, for search on AND/OR graphs that recovers the maximum a posteriori tree of the BCART posterior over decision trees,
    \item Demonstrate that \algname is significantly faster than previous sampling-based approaches,
    \item Demonstrate that the tree recovered by \algname either a) outperforms current state-of-the-art algorithms in performance, or b) demonstrates comparable performance but with smaller trees, and
    \item Provide a heavily optimized C++ implementation that is also callable from Python for practitioners. 
\end{itemize}

\section{Related Work}
\label{sec:related-work}

In this work, we focus on the construction of individual decision trees.
We compare our proposed algorithm with four main classes of prior algorithms: greedy algorithms, ``Optimal'' Decision Trees (ODTs), ``Optimal'' Sparse Decision Trees (OSDTs), and sampling-based approaches.

The most popular method for constructing decision trees is a greedy approach that recursively splits nodes based on a heuristic such as Gini impurity or entropy (in classification) or mean-squared error (in regression) \cite{quinlanInductionDecisionTrees1986}.
However, individual decision trees constructed in this manner often overfit the training data; ensemble methods such as Random Forest and XGBoost attempt to ameliorate overfitting but are significantly more complex than a single decision tree \cite{breimanRandomForests2001, chenXGBoostScalableTree2016}.

So-called ``optimal'' decision trees reformulate the problem of decision tree induction as a global optimization problem, i.e., to find the tree that maximizes global objective function, such as training accuracy, of a given maximum depth \cite{bertsimasOptimalClassificationTrees2017, binoct, verhaegheLearningOptimalDecision2020, aglinLearningOptimalDecision2020}.
Though this problem is NP-Hard in general \cite{HR76}, existing approaches can find the global optimum of shallow trees (depth $\leq 5$) on medium-sized datasets with thousands of datapoints and tens of features.
The original ODT approaches were based on mixed integer programming or binary linear program formulations \cite{verhaegheLearningOptimalDecision2020, Nijssen2007MiningOD, bertsimasOptimalClassificationTrees2017, binoct}.
Other work attempts to improve upon these methods using caching branch-and-bound search \cite{aglinLearningOptimalDecision2020}, constraint programming with AND/OR search \cite{verhaegheLearningOptimalDecision2020}, or dynamic programming with bounds \cite{NEURIPS2022_fe248e22}.
ODTs have been shown to outperform their greedily constructed counterparts with smaller trees \cite{verhaegheLearningOptimalDecision2020, binoct} but still suffer from several drawbacks.
First, choosing the maximum depth hyperparameter is nontrivial, even with cross-validation, and the maximum depth cannot be set too large as the runtime of these algorithms scales exponentially with depth.
Furthermore, ODTs often suffer from overfitting, especially when the maximum depth is set too large.
Amongst ODT approaches, \citet{verhaegheLearningOptimalDecision2020} formulates the search for an optimal decision tree in terms of an AND/OR graph and is most similar to ours, but still suffers from the aforementioned drawbacks.
Additionally, many ODT algorithms exhibit poor anytime behavior \cite{10.1007/978-3-031-26419-1_27}.
Optimal sparse decision trees attempt to adapt ODT approaches to train smaller and sparser trees by incorporating a sparsity penalty in their objectives.
As a result, OSDTs are smaller and less prone to overfitting than ODTs \cite{huOptimalSparseDecision2019, linGeneralizedScalableOptimal2020}.
These approaches, however, often underfit the data \cite{huOptimalSparseDecision2019, linGeneralizedScalableOptimal2020}.

Another class of approaches, called Bayesian Classification and Regression Trees (BCART), introduce a posterior over tree structures given the data and sample trees from this posterior.
Initially, BCART methods were observed to generate better trees than greedy methods \cite{denisonBayesianCARTAlgorithm1998}.
Many variations to the BCART methodology were developed using sampling methods based on Markov-Chain Monte Carlo (MCMC), such as Metropolis-Hastings \cite{pratolaEfficientMetropolisHastings2016} and others \cite{geelsTaxicabSamplerMCMC2022, smc}.
These methods, however, often suffer from exponentially long mixing times in practice and become stuck in local minima \cite{kimMixingRatesBayesian2023}.
In one study, the posterior over trees was represented as a lattice over itemsets \cite{nijssenBayesOptimalClassification2008}.
This approach discovered the maximum a posteriori tree within the hypothesis space of decision trees.
However, this approach required enumerating and storing the entire space of decision trees and therefore placed stringent constraints on the search space of possible trees, based on leaf node support and maximum depth.
Our method utilises the same posterior over tree structures introduced by BCART.
In contrast with prior work, however, we are able to recover the provably maximum a posteriori tree from this posterior in the unconstrained setting.

\section{Preliminaries and Notation}
\label{sec:preliminaries}

In this paper, we focus on the binary classification task, though our techniques extend to multi-class classification and regression.
We also focus on binary datasets, as is common in the decision tree literature \cite{verhaegheLearningOptimalDecision2020, nijssenBayesOptimalClassification2008, Nijssen2007MiningOD} since many datasets can be binarized via bucketing, one-hot encoding, and other techniques.

\textbf{General notation:} We assume we are given a binary dataset $\mathcal{X} \in \{0, 1\}^{N \times F}$ with $N$ samples, $F$ features, and associated binary labels $\mathcal{Y} \in \{0, 1\}^{N}$. 
We let $[u] \coloneqq \{1, \ldots, u\}$, $\mathcal{I} \subseteq [N]$ the indices of a subsample of the dataset, and $(x_i, y_i)$ denote the $i$th sample and its label.
We define $\mathcal{X}|_\mathcal{I} \coloneqq \{x_i : i \in \mathcal{I}\} \subset \mathcal{X}$, $\mathcal{Y}|_\mathcal{I} \coloneqq \{y_i : i \in \mathcal{I}\} \subset \mathcal{Y}$, and $\mathcal{I}|_{f=k} \coloneqq \{i : i \in \mathcal{I}$ and $(x_i)_f = k\}$, for $k \in \{0, 1\}$.
Finally, we let $c^k(\mathcal{I})$ be the count of points in $\mathcal{I}$ with label $k \in \{0, 1\}$, i.e., $c^k(\mathcal{I}) = |\{i: i \in \mathcal{I} \text{ and } y_i = k\}|$ and $\mathcal{V}(\mathcal{I})$ be the set of nontrivial features splits of the samples in $\mathcal{I}$, i.e., the set of features such that neither $\mathcal{I}|_{f=0}$ nor $\mathcal{I}|_{f=1}$ is nonempty.

\textbf{Tree notation:} We let $T = \{n_1, n_2, \dots, n_{M+L}\}$ be a binary classification tree represented as a collection of its nodes and use $n$ to refer to a node in $T$, $m$ to refer to one of the $M$ internal nodes in $T$, and $l$ to refer to one of the $L$ leaf nodes in $T$.
Furthermore, we use $\mathcal{I}(n)$ to denote the indices of the samples in $\mathcal{X}$ that reach node $n$ in $T$, namely $\{i : x_i \in \text{space}(n)\}$, where space$(n)$ is the subset of feature space that reaches node $n$ in $T$.
We also use $c^{k}_l$ to denote the count of points assigned to leaf $l$ with label $k \in \{0, 1\}$ (i.e., $c^{k}_l = c^k(I(l))$), $T_\text{internal} = \{m_1, m_2, \dots, m_M\} \subset T$ to denote the set of internal nodes in tree $T$, and $T_\text{leaves} = \{l_1, l_2, \dots, l_L\} \subset T$ is the set of all leaf nodes in tree $T$.
Finally, we use $d(n)$ to denote the depth of node $n$ in $T$.

\subsection{AND/OR Graph Search}
We briefly recapitulate the concept of AND/OR graphs and a search algorithm for AND/OR graphs, AO*.
AND/OR graph search can be viewed as a generalization of the shortest path problem that allows nodes consisting of independent subproblems to be decomposed and solved separately.
Thus, a solution of an AND/OR graph is not a path but rather a subgraph $\mathcal{S}$ with cost, denoted \texttt{cost}$(\mathcal{S})$, equal to the sum across the costs of its edges.
AND/OR graphs contain two types of nodes: terminal nodes and nonterminal nodes. 
Nonterminal nodes can be further subdivided into AND nodes and OR nodes, with a special OR node designated as the \textit{root} or \textit{start} node $r$.
For a given AND/OR graph $\mathcal{G}$, a \textit{solution graph} $\mathcal{S}$ on an AND/OR graph is a connected subset of nodes of $\mathcal{G}$ in which:

\begin{enumerate}
    \item $r \in \mathcal{S}$,
    \item for every AND node $a \in \mathcal{S}$, \textit{all} the immediate children of $a$ are also in $\mathcal{S}$, and 
    \item for every non-terminal OR node $o \in \mathcal{S}$ \textit{exactly one} of $o$'s children is also in $\mathcal{S}$.
\end{enumerate}

Intuitively, the children of an AND node $a$ represent subtasks that must all be solved for $a$ to be satisfied (e.g., simultaneous prerequisites), and the children of an OR node $o$ represent mutually exclusive satisfying choices.

\begin{figure}
    \centering
    \includegraphics[width=\linewidth]{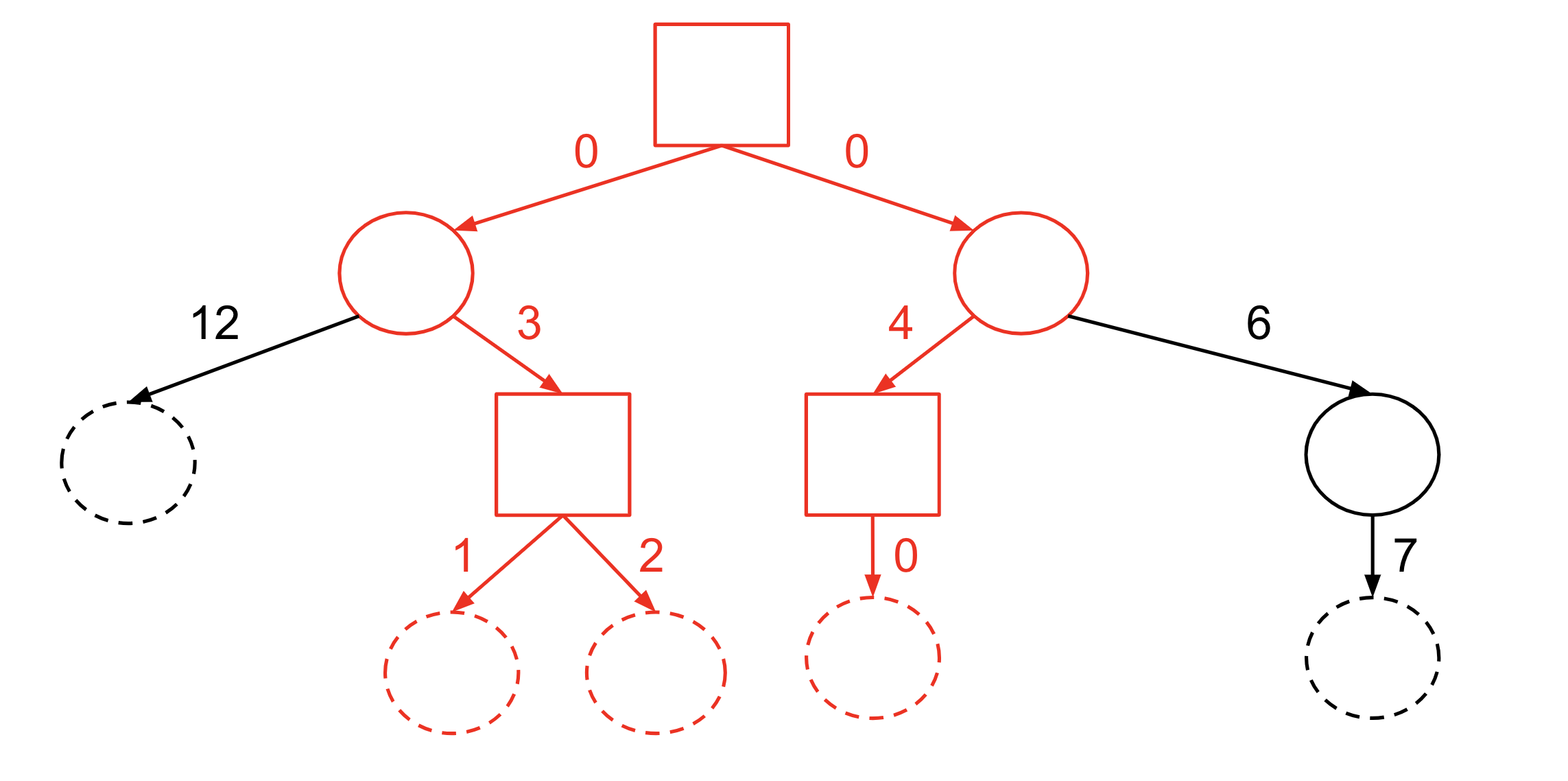}
    \caption{An example (general) AND/OR graph, with AND nodes drawn as squares, and OR nodes drawn as solid circles, and terminal nodes drawn as dashed circles. The minimal cost solution is highlighted in red and has cost $0 + 0 + 3 + 4 + 1 + 2 = 10$. This diagram demonstrates an AND/OR graph where the root node $r$ is an AND node; in \algnamenospace, the root node is an OR node.}
    \label{fig:andor}
\end{figure}

One of the most popular AND/OR graph search algorithms is AO* \cite{mahantiGraphHeuristicSearch1985, AdmissibleHeuristicSearch1983}.
The AO* algorithm explores potential paths in an AND/OR graph in a best-first fashion, guided by a heuristic.
When a new node is explored, its children are revealed and the cost for that node and all of its ancestors is updated; the search then continues.
This process is repeated until the the root node is marked as solved, indicating that no immediately accessible nodes could lead to an increase in heuristic value.
The AO* algorithm is guaranteed to find the minimal cost solution if the heuristic is \textit{admissible}, i.e., the heuristic estimate of cost is always less than or equal to the actual cost of a node.
For more details on the AO* algorithm, we refer the reader to \cite{mahantiGraphHeuristicSearch1985}.
An example AND/OR graph is given in Figure \ref{fig:andor} with its minimal cost solution shown in red.

\textbf{Additional AND/OR graph notation:} In addition to the notation defined above, we use $t$ to refer to a terminal node. When searching over an AND/OR graph, we use $\mathcal{G}$ to refer to the implicit (entire) AND/OR graph and $\mathcal{G}' \subset \mathcal{G}$ to explicit (explored) AND/OR graph, as in prior work.

\subsection{Bayesian Classification and Regression Trees (BCART)}
\label{subsec:bcart}

Bayesian Decision Trees are a family of statistical models of decision trees introduced in \citet{chipmanBayesianCARTModel1998} and \citet{denisonBayesianCARTAlgorithm1998}.
A Bayesian Decision Tree (BDT) is a pair $(T, \Theta)$ where $T$ is a tree and $\Theta = (\theta_{l_1}, \theta_{l_2}, \dots, \theta_{l_L})$ parameterizes the independent probability distributions over labels in the leaf nodes of tree $T$. 
We are interested in the binary classification setting, where each $\theta_l$ parameterizes a Bernoulli distribution $\text{Ber}(\theta_l)$ with $\theta_l \in [0, 1]$.
We denote by $\text{Beta}(\rho^1, \rho^0)$ the Beta distribution with parameters $\rho^1, \rho^0 \in \mathbb{R}^+$ and by $B(c^1, c^0)$ the Beta function.

We note that a BDT's tree $T$ partitions the data such that the sample subsets $\mathcal{I}(l_1)$, $\mathcal{I}(l_2), \mathcal{I}(l_L)$ fall into leaves $l_1, l_2, \dots, l_L$.
Furthermore, a BDT defines a probability distribution over the respective labels occurring in their leaves: each label in leaf $l$ is sampled from Ber$(\theta_l)$).
Every BDT therefore induces a likelihood function, given in Theorem \ref{thm:bdt_likelihood}.

\begin{theorem}
    \label{thm:bdt_likelihood}
    The likelihood of a BDT $(T, \Theta)$ generating labels $\mathcal{Y}$ given features $\mathcal{X}$ is
    \begin{align}
        P(\mathcal{Y} | \mathcal{X}, T, \Theta) &= \prod_{l \in T_\text{leaves}} \prod_{i \in I(l)} \theta_l^{y_i}\left(1 - \theta_l\right)^{1 - y_i} \\
        &= \prod_{l \in T_\text{leaves}} \theta_l^{c^1_l}\left(1 - \theta_l\right)^{c^0_l}
    \end{align}
\end{theorem}

The specific formulation of BCART also assumes a prior distribution over $\Theta$, i.e., that $\theta \sim \text{Beta}(\rho^1, \rho^0)$ for each $\theta \in \Theta$. 
With this assumption, we can derive the likelihood function $P(\mathcal{Y} | \mathcal{X}, T)$; see Theorem \ref{thm:tree_likelihood}.

\begin{theorem}
    \label{thm:tree_likelihood}
    Assume that each $\theta \sim \text{Beta}(\rho^1, \rho^0)$ for each $\theta \in \Theta$. 
    Then the likelihood of a tree $T$ generating labels $\mathcal{Y}$ given features $\mathcal{X}$ is
    \begin{align}
        \label{eqn:tree_likelihood}
        P(\mathcal{Y} | \mathcal{X}, T) = \prod_{l \in T_\text{leaves}} \frac{B(c^1_l + \rho^1, c^0_l + \rho^0)}{B(\rho^1, \rho^0)}
    \end{align}
\end{theorem}

Theorems \ref{thm:bdt_likelihood} and \ref{thm:tree_likelihood} are proven in the appendices; we note they have been observed in different forms in prior work \cite{chipmanBayesianCARTModel1998}.

For notational convenience, we define a leaf count likelihood function $\ell_\text{leaf}(c^1, c^0)$ for integers $c^1$ and $c^0$:
    
\begin{align}
\label{def:leaf-lik}
    \ell_\text{leaf}(c^1, c^0) \coloneqq \frac{B(c^1 + \rho^1, c^0 + \rho^0)}{B(\rho^1, \rho^0)}
\end{align}

and we can rewrite Equation \ref{eqn:tree_likelihood} as 

\begin{align}
\label{eqn:likelihood}
    P(\mathcal{Y} | \mathcal{X}, T) &= \prod_{l \in T_\text{leaves}} \ell_\text{leaf}(c^1_l, c^0_l)
\end{align}

In this work, we utilize the original prior over trees from \cite{chipmanBayesianCARTModel1998}, given in Definition \ref{def:bcart_prior}.

\begin{definition}
    \label{def:bcart_prior}
    The original BCART prior distribution over trees is 
    \begin{align*}
        P(T | \mathcal{X}) &= \left(\prod_{l \in T_\text{leaves}} p_\text{leaf}(d(l), \mathcal{I}(l)) \right) \times \\
        & \hspace{3em} \left(\prod_{m \in T_\text{internal}} p_\text{inner}(d(m), \mathcal{I}(m)) \right) 
    \end{align*}

    where 
    \begin{align}
    \label{def:prior-leaf}
        p_\text{leaf}(d, \mathcal{I}) &= 
        \begin{cases}
            1, & \mathcal{V}(\mathcal{I}) = \emptyset \\
            1 - p_\text{split}(d),  & \mathcal{V}(\mathcal{I}) \neq \emptyset
        \end{cases}
    \end{align}

    \begin{align}
    \label{def:prior-inner}
        p_\text{inner}(d, \mathcal{I}) &= 
        \begin{cases}
            0, & \mathcal{V}(\mathcal{I}) = \emptyset \\
            p_\text{split}(d) / |\mathcal{V}(\mathcal{I})|,  & \mathcal{V}(\mathcal{I}) \neq \emptyset
        \end{cases}
    \end{align}

and 
    
    \begin{align}
    \label{def:prior-split}
        p_\text{split}(d) &= \alpha(1 + d)^{-\beta}
    \end{align}
    
\end{definition}

Intuitively, $p_\text{split}(d)$ is the prior probability of any node splitting and is allocating equally amongst valid splits.
This choice of prior, $P(T | \mathcal{X})$, combined with the likelihood function in Equation \ref{eqn:likelihood} induces the posterior distribution over trees $P(T | \mathcal{Y}, \mathcal{X})$:

\begin{equation}
    P(T | \mathcal{Y}, \mathcal{X}) \propto P(\mathcal{Y} | \mathcal{X}, T) P(T | \mathcal{X})
\end{equation}

Throughout our analysis, we treat the dataset $(\mathcal{X}, \mathcal{Y})$ as fixed.

\begin{figure}
    \centering
    \includegraphics[width=\linewidth]{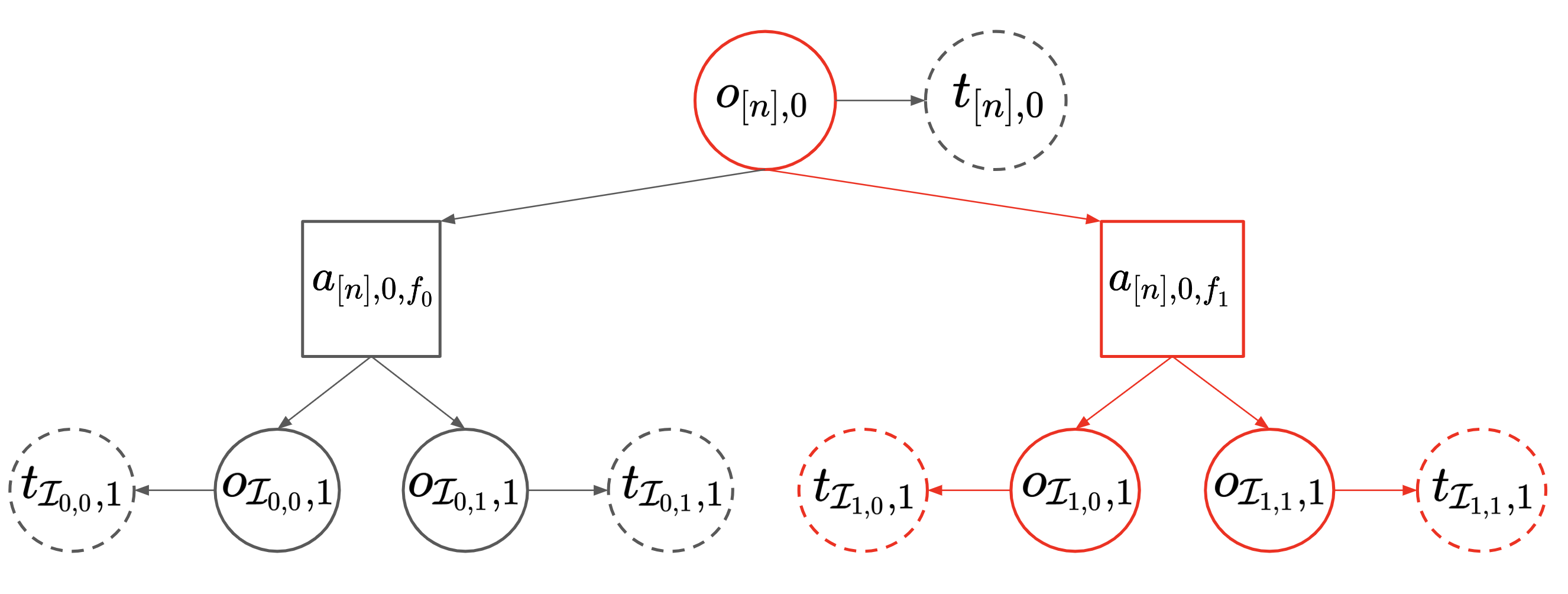}    
    \caption{Example of the defined BCART AND/OR graph $\mathcal{G}_{\mathcal{X}, \mathcal{Y}}$. OR nodes are represented as circles with solid borders, terminal nodes as circles with dashed borders, and AND nodes as squares. In this dataset, two feature splits are possible at the root node ($f_0$ and $f_1$) and no further splits are possible at deeper nodes. The best solution on this AND/OR graph is highlighted in red and corresponds with a stump which splits the root node, corresponding to the entire dataset, on feature $f_1$.}
    \label{fig:bcart_andor}
\end{figure}

\section{Connecting BCART with AND/OR Graphs}
\label{sec:bdts_and_ao_graphs}

Given a dataset $(\mathcal{X}, \mathcal{Y})$, we will now construct a special AND/OR graph $\mathcal{G}_{\mathcal{X}, \mathcal{Y}}$.
We will then show that a minimal cost solution graph on $\mathcal{G}_{\mathcal{X}, \mathcal{Y}}$ corresponds directly with the maximum a posteriori tree given our choice of prior distributions $P(T | \mathcal{X})$ and $P(\Theta)$.
Using this construction, the problem of finding the maximum a posteriori tree of our posterior is reduced to that of finding the minimum cost solution graph on $\mathcal{G}_{\mathcal{X}, \mathcal{Y}}$. 

\begin{definition}[BCART AND/OR graph $\mathcal{G}_{\mathcal{X}, \mathcal{Y}}$]
    \label{def:bcart-andor-graph}
    Given a dataset $(\mathcal{X}, \mathcal{Y})$, construct the AND/OR graph $\mathcal{G}_{\mathcal{X}, \mathcal{Y}}$ as follows:

    \begin{enumerate}
        \item For every possible subset $\mathcal{I} \subset [N]$ and depth $d \in \{0, \dots, F\}$, create an OR node $o_{\mathcal{I}, d}$.
        \item For every OR node $o_{\mathcal{I}, d}$ created in Step 1, create a terminal node $t_{\mathcal{I}, d}$ and draw an edge from $o_{\mathcal{I}, d}$ to $t_{\mathcal{I}, d}$ with cost $\texttt{cost}(o_{\mathcal{I}, d}, t_{\mathcal{I}, d}) = -\log p_\text{leaf}(d, \mathcal{I}) - \log \ell_\text{leaf}(c^1(\mathcal{I}), c^0(\mathcal{I}))$.
        \item For every OR node $o_{\mathcal{I}, d}$ created in Step 1, create $F$ AND nodes $a_{\mathcal{I}, d, 1}, \ldots, a_{\mathcal{I}, d, F}$ and drawn an edge from $o_{\mathcal{I}, d}$ to each $a_{\mathcal{I}, d, f}$ with cost $\texttt{cost}(o_{\mathcal{I}, d}, a_{\mathcal{I}, d, f}) = -\log p_\text{inner}(d)$.
        \item For every pair $a_{\mathcal{I}, d, f}$ and $o_{\mathcal{I}', d+1}$ where $\mathcal{I}|_{f=k} = \mathcal{I}'$ for some $f \in [F]$ and $k \in \{0, 1\}$, draw an edge from $a_{\mathcal{I}, d, f}$ to $o_{\mathcal{I'}, d+1}$ with cost $\texttt{cost}(a_{\mathcal{I}, d, f}, o_{\mathcal{I'}, d+1}) = 0$.
        \item Let $o_{[n], 0}$, the OR node representing all sample indices, be the unique root node $r$ of $\mathcal{G}_{\mathcal{X}, \mathcal{Y}}$.
        \item Remove all OR nodes representing empty subsets and their neighbors.
        \item Remove all nodes not connected to the root node $r$.
    \end{enumerate}
\end{definition}

We note that $\mathcal{G}_{\mathcal{X}, \mathcal{Y}}$ contains $F\times2^n$ OR Nodes, $F\times2^n$ terminal nodes (one for each OR Node), and  $F^2\times2^n$ AND nodes ($F$ for each OR Node) and so is finite.

Intuitively, each OR node $o_{\mathcal{I}, d}$ in $\mathcal{G}_{\mathcal{X}, \mathcal{Y}}$ corresponds with the subproblem of discovering a maximum a posteriori subtree starting from depth $d$ and over the subset of samples $\mathcal{I}$ from dataset $\mathcal{X}, \mathcal{Y}$.
Each AND node $a_{\mathcal{I}, d, f}$ then represents the same subproblem but given that a decision was already made to split on feature $f$ at the root node of this subtree.
A valid solution graph on $\mathcal{G}_{\mathcal{X}, \mathcal{Y}}$ corresponds with a binary classification tree $T$ on the dataset $(\mathcal{X}, \mathcal{Y})$ and the value of a solution is related to the posterior probability of $T$ given by $P(T | \mathcal{Y}, \mathcal{X})$.
We formalize these properties in Theorems \ref{thm:bijection} and \ref{thm:andor_propto_tree_ll}.

\begin{theorem}
\label{thm:bijection}
    Every solution graph on AND/OR graphs induces a unique binary decision tree.
    Furthermore, every decision tree can be represented as a unique solution graph under this correspondence.
    Thus, there is natural bijection between solution graphs on $\mathcal{G}_{\mathcal{X}, \mathcal{Y}}$ and binary decision trees.
\end{theorem}

\begin{theorem}
\label{thm:andor_propto_tree_ll}
     Under the natural bijection described in Theorem \ref{thm:bijection}, given a solution graph $\mathcal{S}$ and its corresponding tree $T$, we have that $\texttt{cost}(\mathcal{S}) = -\log P(T, \mathcal{Y} | \mathcal{X})$. 
     Therefore the minimal cost solution over $\mathcal{G}_{\mathcal{X}, \mathcal{Y}}$ corresponds with a maximum a posteriori tree.
\end{theorem}

The bijection constructed in Theorems \ref{thm:bijection} and \ref{thm:andor_propto_tree_ll} is depicted in Figure \ref{fig:bcart_andor_solution_to_tree}.
Due to space constraints, we defer a formal description of this bijection to Appendix \ref{app:proofs}.

\begin{figure}
    \centering
    \includegraphics[width=\linewidth]{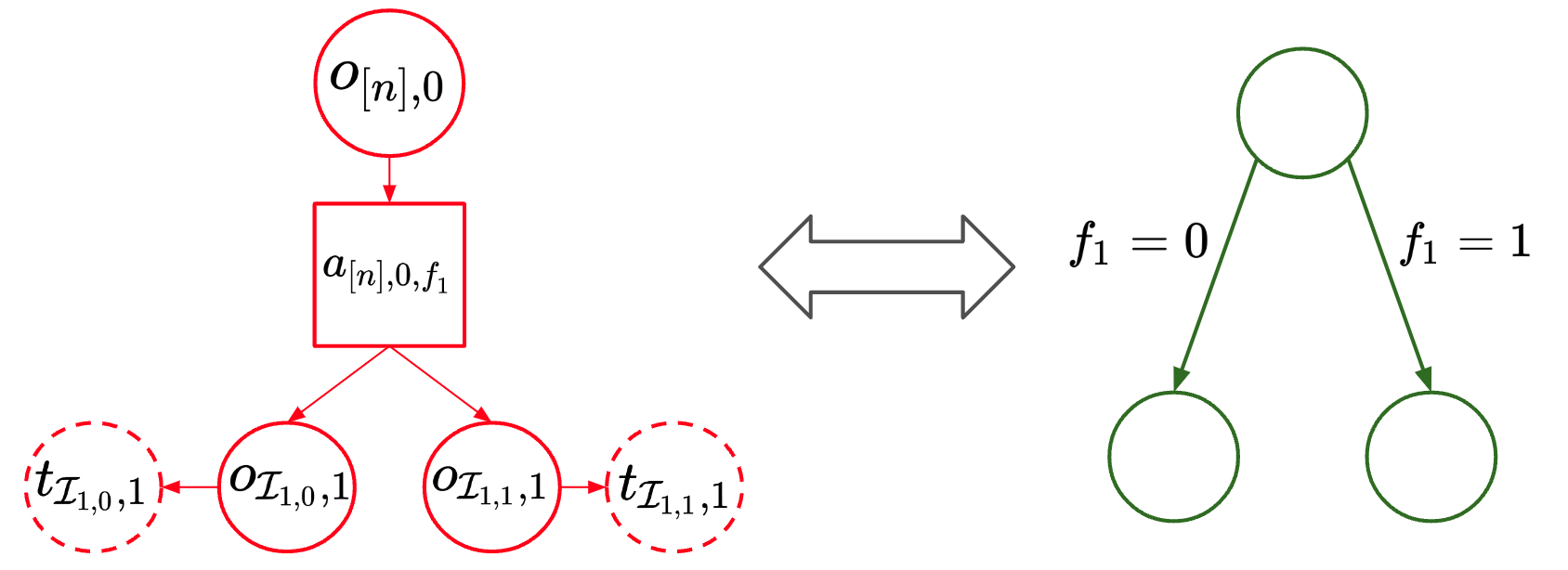}    
    \caption{Example map between an example solution of the AND/OR graph $\mathcal{G}_{\mathcal{X}, \mathcal{Y}}$ depicted in Figure \ref{fig:bcart_andor} and its corresponding binary classification tree. We see that the resulting tree is a stump which splits on feature $f_1$ at the root.}
    \label{fig:bcart_andor_solution_to_tree}
\end{figure}

\section{\algname}
\label{sec:methods}

\begin{algorithm}[tb]
\caption{\texttt{\algnamenospace}}
\label{alg:maptree}
\textbf{Input}: Root OR Node $r$, cost function $\texttt{cost}$, and heuristic function $h$ for AND/OR graph $\mathcal{G}$ \\
\textbf{Output}: Solution graph $\mathcal{S}$ 
\begin{algorithmic}[1] 
\STATE $\mathcal{G}' := \{r\}$
\STATE $\mathcal{E} := \emptyset$
\STATE $LB[r] := h(r)$
\STATE $UB[r] := \infty$
\WHILE{$LB[r] < UB[r]$ and \texttt{time remaining}}
    \STATE $o :=$ \texttt{findNodeToExpand}($r$, $\texttt{cost}$, $\mathcal{E}$, $LB$, $UB$)
    \STATE Let $t$ be the terminal node child of $o$
    \STATE $\mathcal{E} := \mathcal{E} \cup \{o\}$
    \STATE $\mathcal{G}' := \mathcal{G}' \cup \{t\}$
    \STATE Let $\{a_1, \dots, a_{F}\}$ be the AND node children of $o$
    \FORALL{$a_f \in \{a_1, \dots, a_{F}\}$}
        \STATE Let $\{o_{f=0}, o_{f=1}\}$ be the OR node children of $a_f$
        \STATE $LB[o_{f=0}] := h(o_{f=0})$
        \STATE $LB[o_{f=1}] := h(o_{f=1})$
        \STATE $v^{(lb)}_{f=0} = \texttt{cost}(a_f, o_{f=0}) + h(o_{f=0})$
        \STATE $v^{(lb)}_{f=1} = \texttt{cost}(a_f, o_{f=1}) + h(o_{f=1})$
        \STATE $LB[a_f] := v_{f=0} + v_{f=1}$
        \STATE $\mathcal{G}' := \mathcal{G}' \cup \{a, o_{f=0}, o_{f=1}\}$
    \ENDFOR
    \STATE \texttt{updateLowerBounds}($o$, $\texttt{cost}$, $LB$)
    \STATE \texttt{updateUpperBounds}($o$, $\texttt{cost}$, $UB$)
\ENDWHILE
\RETURN \texttt{getSolution}($r$, $\texttt{cost}$, $UB$)
\end{algorithmic}
\end{algorithm}

\begin{algorithm}[tb]
\caption{\texttt{getSolution}}
\label{alg:getSolution}
\textbf{Input}: ORNode $o$, cost function $\texttt{cost}$, and upper bounds $UB$ \\
\textbf{Output}: Solution graph $\mathcal{S}$
\begin{algorithmic}[1] 
\STATE Let $\{a_1, \dots, a_{F}\}$ be the AND node children of $o$
\STATE Let $t$ be the terminal node child of $o$
\STATE $a_{f^*} := \arg\min_{c \in \{a_1, \dots, a_{F}\}} \left(\texttt{cost}(o, c) + UB[c]\right)$
\STATE $\mathcal{S} := \{o\}$
\IF{$\texttt{cost}(o, t) + UB[t] \leq \texttt{cost}(o, a_{f^*}) + UB[a_{f^*}]$}
    \STATE $\mathcal{S} := \mathcal{S} \cup \{t\}$
\ELSE
    \STATE Let $o_{f^*=0}, o_{f^*=1}$ be the children of $a_{f^*}$
    \STATE $\mathcal{S} := \mathcal{S} \cup \{a_{f^*}\}$
    \STATE $\mathcal{S}_0 :=$ \texttt{getSolution}($o_{f^*=0}$)
    \STATE $\mathcal{S}_1 :=$ \texttt{getSolution}($o_{f^*=1}$)
    \STATE $\mathcal{S} := \mathcal{S} \cup \mathcal{S}_0 \cup \mathcal{S}_1$
\ENDIF
\RETURN $\mathcal{S}$
\end{algorithmic}
\end{algorithm}

\begin{algorithm}[tb]
\caption{\texttt{findNodeToExpand}}
\label{alg:findNodeToExpand}
\textbf{Input}: Root node $r$, cost function $\texttt{cost}$, set of expanded nodes $\mathcal{E}$, lower bounds $LB$ and upper bounds $UB$ \\
\textbf{Output}: ORNode $o$
\begin{algorithmic}[1] 
\STATE $o := r$
\WHILE{$o \in \mathcal{E}$}
    \STATE Let $\{a_1, \dots, a_{F}\}$ be the children of $o$
    \STATE $a^* := \arg\min_{c \in \{a_1, \dots, a_{F}\}} \left(\texttt{cost}(o, c) + LB[c]\right)$
    \STATE Let $o^*_0, o^*_1$ be the children of $a^*$
    \IF{$UB[o^*_0]$ - $LB[o^*_0] > UB[o^*_1]$ - $LB[o^*_1]$}
        \STATE $o := o^*_0$
    \ELSE
        \STATE $o := o^*_1$
    \ENDIF
\ENDWHILE
\RETURN $o$
\end{algorithmic}
\end{algorithm}

\begin{algorithm}[tb]
\caption{\texttt{updateLowerBounds}}
\label{alg:updateLowerBounds}
\textbf{Input}: ORNode $l$, cost function $\texttt{cost}$, lower bounds $LB$
\begin{algorithmic}[1] 
\STATE $\mathcal{V} = \{l\}$
\WHILE{$|V| > 0$}
    \STATE Remove a node $o$ from $\mathcal{V}$ with maximal depth
    \STATE Let $\{a_1, \dots, a_{F}\}$ be the AND node children of $o$
    \STATE Let $t$ be the terminal node child of $o$
    \STATE $v^{(lb)}_\text{split} = \min_{c \in \{ a_1, \dots, a_{F} \}} \left(\texttt{cost}(o, c) + LB[c]\right)$
    \STATE $v^{(lb)} = \min \{ v^{(lb)}_\text{split}, \texttt{cost}(o, t) \}$
    \IF {$v^{(lb)} > LB[o]$}
        \STATE $LB[o] := v^{(lb)}$
        \STATE Add all parents of $o$ to $\mathcal{V}$
    \ENDIF
\ENDWHILE
\end{algorithmic}
\end{algorithm}

\begin{algorithm}[tb]
\caption{\texttt{updateUpperBounds}}
\label{alg:updateUpperBounds}
\textbf{Input}: ORNode $l$, cost function $\texttt{cost}$, upper bounds $UB$
\begin{algorithmic}[1] 
\STATE $\mathcal{V} = \{l\}$
\WHILE{$|V| > 0$}
    \STATE Remove a node $o$ from $\mathcal{V}$ with maximal depth
    \STATE Let $\{a_1, \dots, a_{F}\}$ be the AND node children of $o$
    \STATE Let $t$ be the terminal node child of $o$
    \STATE $v^{(ub)}_\text{split} = \min_{c \in \{ a_1, \dots, a_{F} \}} \left(\texttt{cost}(o, c) + UB[c]\right)$
    \STATE $v^{(ub)} = \min \{ v^{(ub)}_\text{split}, \texttt{cost}(o, t) \}$
    \IF {$v^{(ub)} < UB[o]$}
        \STATE $UB[o] := v^{(ub)}$
        \STATE Add all parents of $o$ to $\mathcal{V}$
    \ENDIF
\ENDWHILE
\end{algorithmic}
\end{algorithm}

Theorems \ref{thm:bijection} and \ref{thm:andor_propto_tree_ll} imply that it is sufficient to find the minimum cost solution graph on $\mathcal{G}_{\mathcal{X}, \mathcal{Y}}$ to recover the MAP tree under the BCART posterior. 
In this section, we introduce \algnamenospace, an AND/OR search algorithm that finds a minimal cost solution on $\mathcal{G}_{\mathcal{X}, \mathcal{Y}}$. 
\algname is shown in Algorithm \ref{alg:maptree}.

A key component of \algname is the Perfect Split Heuristic $h$ that guides the search, presented in Definition \ref{def:heuristic}.
\begin{definition}[Perfect Split Heuristic]
\label{def:heuristic}
    For OR node $o_{\mathcal{I}, d}$ with terminal node child $t_{\mathcal{I}, d}$, let
    \begin{align}
        h(o_{\mathcal{I}, d}) &= -\max \{ \\
        &\log \ell_\text{leaf}(c^1(\mathcal{I}), c^0(\mathcal{I})), \\
        &\log p_\text{split}(d, \mathcal{I}) \\
        & + \log \ell_\text{leaf}(c^1(\mathcal{I}), 0) \\
        & + \log \ell_\text{leaf}(0, c^0(\mathcal{I}))\}
    \end{align}
    and for AND node $a_{\mathcal{I}, d, f}$ with OR node children $o_{\mathcal{I}|_{f=0}, d + 1}$ and $o_{\mathcal{I}|_{f=1}, d + 1}$, let
    \begin{align}
        h(a_{\mathcal{I}, d, f}) &= h(o_{\mathcal{I}|_{f=0}, d + 1}) + h(o_{\mathcal{I}|_{f=1}, d + 1})
    \end{align}
\end{definition}

Intuitively, the Perfect Split Heuristic describes the negative log posterior probability of the best potential subtree rooted at the given OR node $o_{\mathcal{I}, d}$: one that perfectly classifies the data in a single additional split.
The heuristic guides the search away from subproblems that are too deep or for which the labels have already been poorly divided.
We prove that this heuristic is a lower bound (admissible) and consistent in later sections.

\subsection{Analysis of \algname}

We now introduce several key properties of \algnamenospace. 
In particular, we show that (1) the Perfect Split Heuristic is consistent and therefore also admissible, (2) \algname finds the maximum a posteriori tree of the BCART posterior upon completion, and (3) upon early termination, \algname returns the minimum cost solution within the explored explicit graph $\mathcal{G}'$.
Theorems \ref{thm:consistency_of_h} - \ref{thm:anytime} and Corollary \ref{cor:recover_map} are proven in Appendix \ref{app:proofs}.

\begin{theorem}[Consistency of the Perfect Split Heuristic]
\label{thm:consistency_of_h}
    The Perfect Split Heuristic in Definition \ref{def:heuristic} is consistent, i.e., for any OR node $o$ with children $\{t, a_1, \dots, a_F\}$:
    \begin{align}
        h(o) &\leq \min_{c \in \{t, a_1, \dots, a_F\}} \texttt{cost}(o, c) + h(c)
    \end{align}
    and for any AND node $a$ with children $\{o_0, o_1\}$:
    \begin{align}
        h(a) &\leq \sum_{c \in \{o_0, o_1\}} \texttt{cost}(a, c) + h(c)
    \end{align}
\end{theorem}

\begin{theorem}[Finiteness of \algnamenospace]
\label{thm:finiteness}
    Algorithm \ref{alg:maptree} always terminates.
\end{theorem}

\begin{theorem}[Correctness of \algnamenospace]
\label{thm:correctness}
    When Algorithm \ref{alg:maptree} does not terminate early due to the \texttt{time remaining} condition, it always outputs a minimal cost solution on $\mathcal{G}_{\mathcal{X}, \mathcal{Y}}$ upon completion.
\end{theorem}

\begin{corollary}
\label{cor:recover_map}
    Consider the tree induced by the output of Algorithm \ref{alg:maptree} under the natural bijection described in Section \ref{sec:bdts_and_ao_graphs}.
    By Theorems \ref{thm:bijection} and \ref{thm:andor_propto_tree_ll}, this tree is the maximum a posteriori tree $\arg\max_T P(T|\mathcal{X},\mathcal{Y})$.
\end{corollary}

\begin{theorem}[Anytime optimality of \algnamenospace]
\label{thm:anytime}
    Upon early termination, Algorithm \ref{alg:maptree} outputs the minimal cost solution across the explicit subgraph $\mathcal{G}'$ of already explored nodes.
\end{theorem}


\section{Experiments}
\label{sec:exps}

We evaluate the performance of \algname in multiple settings.
In all experiments in this section, we set $\alpha = 0.95$ and $\beta = 0.5$.
We find that our results are not highly dependent on the choices of $\alpha$ and $B$; see Appendix \ref{app:add_experiments}.

In the first setting, we compare the efficiency of \algname to the Sequential Monte Carlo (SMC) and Markov-Chain Monte Carlo (MCMC) baselines from \citet{smc} and \citet{chipmanBayesianCARTModel1998}, respectively.
In the second setting, we create a synthetic dataset in which the true labels are generated by a randomly generated tree and measure generalization performance with respect to training dataset size.
In the third setting, we measure the generalization accuracy, log likelihood, and tree size of models generated by \algname and baseline algorithms across all 16 datasets from the CP4IM dataset repository \cite{gunsItemsetMiningConstraint2011}.

\subsection{Speed Comparisons against MCMC and SMC}
\label{subsec:speed}

We first compare the performance of \algname with the SMC and MCMC baselines from \citet{smc} and \citet{chipmanBayesianCARTModel1998}, respectively, on all 16 binary classification datasets from the CP4IM dataset repository \cite{gunsItemsetMiningConstraint2011}.
We note that all three methods, given infinite exploration time, should recover the maximum a posteriori tree from the BCART posterior.
However, it has been observed that the mixing times for Markov-Chain-based methods, such as the MCMC and SMC baselines, is exponential in the depth of the data-generating tree \cite{kimMixingRatesBayesian2023}.
Furthermore, the SMC and MCMC methods are unable to determine when they have converged, nor can they provide a certificate of optimality upon convergence.

In our experiments, we modify the hyperparameters of each algorithm and measure the training time and log posterior of the data under the output tree (Figure \ref{fig:speed}).
In 12 of the 16 datasets in Figure \ref{fig:synthetic}, \algname outperforms SMC and MCMC and is able to find trees with higher log posterior faster than the baseline algorithms.
Furthermore, in 5 of the 16 datasets, \algname converges to the provably optimal tree, i.e., the maximum a posteriori tree of the BCART posterior. 

\begin{figure}
    \centering
    \includegraphics[width=\linewidth]{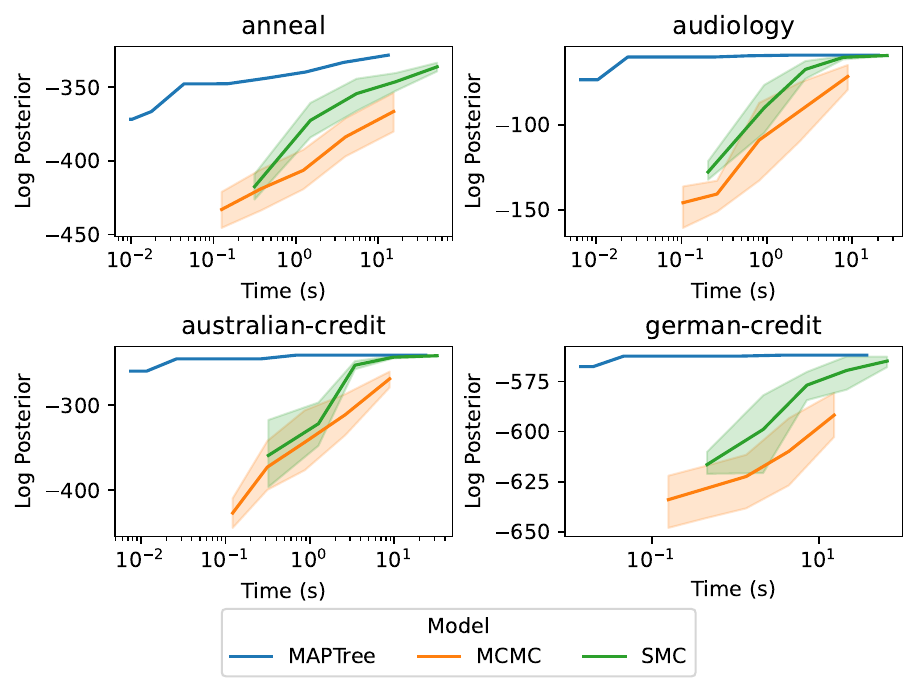}
    \caption{Comparison of \algnamenospace, SMC, and MCMC on 4 datasets (results for an additional 12 datasets are presented in the appendix). Curves are created by modifying the hyperparameters for each algorithm and measuring training time and log posterior of the data under the tree. Higher and further left is better, i.e., better log posteriors in less time. In 12 of the 16 datasets, \algname outperforms SMC and MCMC and is able to find trees with higher log posterior faster than the baseline algorithms. Furthermore, in 5 of the 16 datasets, \algname converges to the provably optimal tree, i.e., the MAP tree. 95\% confidence intervals are derived by bootstrapping the results of 10 random seeds and time is averaged across the 10 seeds.}
    \label{fig:speed}
\end{figure}

\begin{figure}
    \centering
    \includegraphics[width=\linewidth]{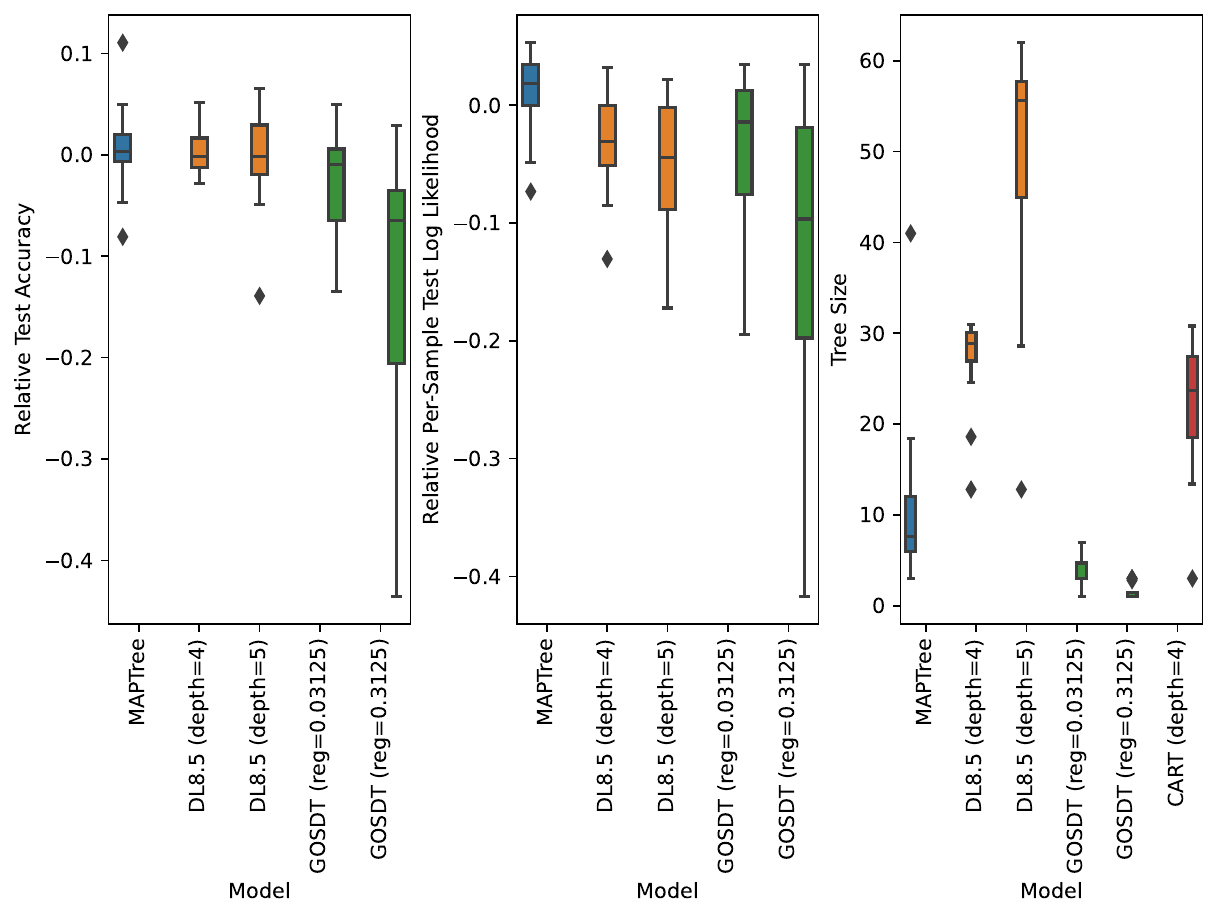}
    \caption{Box-and-whisker plot of stratified 10-fold relative test accuracy, relative per-sample log likelihood, and size of the output tree for \algname and various baseline algorithms for each of the 16 CP4IM datasets. Test accuracy and log likelihood are relative to that of CART (max depth $4$). Higher is better for the left and center plots and lower is better for the right plot. Against all baseline algorithms, \algname either a) performs better in test accuracy or log likelihood, or b) performs comparably in test accuracy and log likelihood but produces smaller trees.}
    \label{fig:ac_ll_size}
\end{figure}

\subsection{Fitting a Synthetic Dataset}
\label{subsec:synthetic}

We measure the generalization performance of \algname and various other baseline algorithms as a function of training dataset size on tree-generated data. 

\textbf{Synthetic Data}: We construct a synthetic dataset where labels are generated by a randomly generated tree. We first construct a random binary tree structure as specified in \citet{Devroye1995TheBB} via recursive random divisions of the available internal nodes to the left or right subtree.
Next, features are selected for each internal node uniformly at random such that no internal node splits on the same feature as its ancestors.
Lastly, labels are assigned to the leaf nodes in alternating fashion so as to avoid compression of the underlying tree structure.
Individual datapoints with 40 features are then sampled with each feature drawn i.i.d. from Ber$(1/2)$, and their labels are determined by following the generated tree to a leaf node.
We repeat this process 20 times, generating 20 datasets for 20 random trees.
We also randomly flip $\epsilon$ of the training data labels, with $\epsilon$ ranging from $0$ to $0.25$ to simulate label noise.

In our experiments, \algname generates trees which outperform both the greedy, top-down approaches and ODT methods in test accuracy for various training dataset sizes and values of label corruption proportion $\epsilon$; the results are presented in Figure \ref{fig:synthetic} in Appendix \ref{app:add_experiments} due to space constraints.
We note that though some baseline algorithms demonstrate comparable performance at a single noise level, no baseline algorithm demonstrates test accuracy comparable to \algname across all noise levels.
We also emphasize that \algname requires no hyperparameter tuning, whereas we experimented with various values of hyperparameters for the baseline algorithms in which performance was highly dependent on hyperparameter values (e.g., DL8.5 and GOSDT); see Appendix \ref{app:add_experiments}.

\subsection{Accuracy, Likelihood, and Size Comparisons on Real World Benchmarks}

We also compare the accuracy, test log likelihood, and sizes of trees generated by \algname and baseline algorithms on 16 real world binary classification datasets from the CP4IM dataset repository \cite{gunsItemsetMiningConstraint2011} (Figure \ref{fig:ac_ll_size}). 
Against all baseline algorithms, \algname either a) performs better in test accuracy or log likelihood, or b) performs comparably in test accuracy and log likelihood but produces smaller trees.

\section{Discussion and Conclusions}
\label{sec:disc}

We presented \algnamenospace, an algorithm which provably finds the maximum a posteriori tree of the BCART posterior for a given dataset.
Our algorithm is inspired by best-first-search algorithms over AND/OR graphs and the observation that the search problem for trees can be framed as a search problem over an appropriately constructed AND/OR graph.

\algname outperforms thematically similar approaches such as SMC- and MCMC-based algorithms, finding higher log-posterior trees faster, and is able to determine when it has converged to the maximum a posteriori tree, unlike prior work.
\algname also outperforms greedy, ODT, and ODST construction methods in test accuracy on the synthetic dataset constructed in Section \ref{sec:exps}.
Furthermore, on many real world benchmark datasets, \algname either a) demonstrates better generalization performance, or b) demonstrates comparable generalization performance but with smaller trees.

A limitation of \algname is that it constructs a potentially large AND/OR graph, which consumes a significant amount of memory.
We leave optimizations that may permit \algname to run on huge datasets to future work.
Nonetheless, with the optimizations presented in Section \ref{sec:exps}, we find that \algname was performant enough to run on the CP4IM benchmark datasets used in evaluation of previous ODT benchmarks.

\section*{Acknowledgements}
\label{sec:acknowledgements}

We would like to thank the anonymous reviewers and Area Chair for their reviews and helpful feedback.

M. T. was funded by a J.P. Morgan AI Fellowship, a Stanford Indisciplinary Graduate Fellowship, a Stanford Data Science Scholarship, and an Oak Ridge Institute for Science and Engineering Fellowship.

\appendix

\bibliography{MAPTree-biblatex}

\clearpage
\section{Proofs of Theorems}
\label{app:proofs}

In this section, we present the proofs of the theorems.

\begin{proof}[Proof of Theorem \ref{thm:bdt_likelihood}]
We note that this result was first presented in the original BCART paper for the more general classification setting \cite{chipmanBayesianCARTModel1998}.
We reproduce the proof below in the binary classification setting. 

By the definition of a BDT $(T, \Theta)$, the tree $T$ partitions the data such that the sample subsets $I(l_1)$, $I(l_2), I(l_L)$ fall into leaves $l_1, l_2, \dots, l_L$ and each leaf contains an independent probability distribution Ber$(\theta_l)$ that governs the probability of a given label occurring in each respective leaf.
Note that the probability of each label $y_i$ occurring is conditionally independent of the other elements of $\mathcal{Y}$ and the dataset $\mathcal{X}$ given its leaf $l$ (which is determined only from the tree structure $T$ and $x_i$) and the corresponding parameter $\theta_l$ (which is determined from the global parametrization $\Theta$ and $l$). Therefore

\begin{align}
P(\mathcal{Y} | \mathcal{X}, T, \Theta) &= \prod_{j \in [N]} P(y_j | x_j, T, \Theta) \\
&= \prod_{l \in T_\text{leaves}} \prod_{i \in \mathcal{I}(l)} P(y_i | \theta_l) \\
&= \prod_{l \in T_\text{leaves}} \prod_{i \in \mathcal{I}(l)} \theta_l^{y_i} (1-\theta_l)^{1-y_i} \\
&= \prod_{l \in T_\text{leaves}} \theta_l^{c_l^1} (1-\theta_l)^{c_l^0}
\end{align}
\end{proof}

\begin{proof}[Proof of Theorem \ref{thm:tree_likelihood}]
    The likelihood of a tree $T$ generating labels $\mathcal{Y}$ given features $\mathcal{X}$ can be obtained by marginalizing over $\Theta$ using the prior $P(\Theta)$ given in Section \ref{sec:preliminaries}:
    \begin{align}
        P(\mathcal{Y} | \mathcal{X}, T) &= \int_{\Theta} P(\mathcal{Y} | \mathcal{X}, T, \Theta) P(\Theta) d\Theta \\
        &= \int_\Theta  \prod_{l \in T_\text{leaves}} \left( \theta_l^{c^1_l}\left(1 - \theta_l\right)^{c^0_l} \right) \times\\
        &\left( \frac{\theta_l^{\rho^1}(1 - \theta_l)^{\rho^0}}{B(\rho^1, \rho^0)} \right) d\Theta \\
        &= \int_\Theta  \prod_{l \in T_\text{leaves}} \frac{1}{B(\rho^1, \rho^0)} \times \\
        &\left( \theta_l^{c^1_l + \rho^1}\left(1 - \theta_s\right)^{c^0_l + \rho^0} \right) d\Theta \\
        &= \prod_{l \in T_\text{leaves}} \frac{1}{B(\rho^1, \rho^0)} \times   \\
        &\int_{\theta_l} \left( \theta_l^{c^1_l + \rho^1}\left(1 - \theta_j\right)^{c^0_l + \rho^0} \right) d\Theta \\
        &= \prod_{l \in T_\text{leaves}} \frac{B(c^1_l + \rho^1, c^0_l + \rho^0)}{B(\rho^1, \rho^0)}
    \end{align}
    where for the second equality we used Theorem \ref{thm:bdt_likelihood} and the choice of prior $P(\Theta)$, and the definition of the Beta function $B(\rho^1, \rho^0)$ throughout.
\end{proof}

\begin{proof}[Proof of Theorem \ref{thm:bijection}]
    Our proof is by construction; we explicitly define the ``natural'' mapping both ways.
    We first show that any binary decision tree corresponds to a solution graph in $\mathcal{G}_{\mathcal{X}, \mathcal{Y}}$ under a natural correspondence, which we explicitly construct.
    Given a binary decision tree $T$ with $m$ nodes (some of which may be internal and some of which are leaf nodes), we construct its solution graph $\mathcal{S} \subset \mathcal{G}_{\mathcal{X}, \mathcal{Y}}$ explicitly.
    
    Let $T_\text{leaves}$ and $T_\text{internal}$ denote the leaf and internal nodes of $T$, respectively.
    For node $n$ with depth $d(n)$ in $T$, denote by $\mathcal{I}(n)$ the datapoint indices for the datapoints which reach node $n$.
    Furthermore, if $n$ is an internal node of $T$, let $f(n)$ denote the feature on which $n$ is split into its children in $T$.
    Then let $\mathcal{S} = \{o_{\mathcal{I}(n), d(n)} : n \in T\} \cup \{t_{\mathcal{I}(n), d(n)} : n \in T_\text{leaves}\} \cup \{a_{\mathcal{I}, d(n), f(n)} : n \in T_\text{internal}\} \cup \{o_{\mathcal{I}(n)|_{f(n) = k}, d(n)+1} : n \in T_\text{internal} \text{ and } k \in \{0, 1\}\}$. 

    We will now show that this choice of $\mathcal{S}$ is a solution graph on $\mathcal{G}_{\mathcal{X}, \mathcal{Y}}$.
    First note that the root node in $T$ has depth $0$ and must exist and contain the whole dataset, so $o_{[n], 0} \in \mathcal{S}$.
    Now consider any AND node $a_{\mathcal{I}, d(n), f(n)} \in \mathcal{S}$.
    By construction, we must have that both $o_{\mathcal{I}(n)|_{f(n) = 0}, d(n)+1} \in \mathcal{S}$ and $o_{\mathcal{I}(n)|_{f(n) = 1}, d(n)+1} \in \mathcal{S}$, so all of the immediate children of $a_{\mathcal{I}, d(n), f(n)} \in \mathcal{S}$ must also be in $\mathcal{S}$.
    Finally, consider any OR node $o \in \mathcal{S}$. We must have one of three cases. Either:
    \begin{enumerate}
        \item $o$ is of the form $o_{\mathcal{I}(n), d(n)}$ for $n \in T_\text{leaves}$, in which case $t_{\mathcal{I}(n), d(n)} \in \mathcal{S}$,
        \item $o$ is of the form $o_{\mathcal{I}(n), d(n)}$ for $n \in T_\text{internal}$, in which case $a_{\mathcal{I}, d(n), f(n)} \in \mathcal{S}$,
        \item $o$ is of the form $o_{\mathcal{I}(n)|_{f(n) = k}, d(n)+1}$ for some node $n \in T_\text{internal}$ and some $k$. In this case, since $n \in T_\text{internal}$, $n$ must have children $n_k$ (for $k = 0, 1$) in $T$ obtained by splitting node $n$ on feature $f(n)$ and so we must have that $o$ is of the form $o_{\mathcal{I}(n_k), d(n_k)}$ for some $n_k \in T$. In this case, we can apply either Case 1 or Case 2 to $n_k$ to show that $o$ must have exactly one child in $\mathcal{S}$.
    \end{enumerate}

    In all cases, each OR node $o \in \mathcal{S}$ must have exactly one child in $\mathcal{S}$.
    Thus, $\mathcal{S}$ is a solution graph on $\mathcal{G}_{\mathcal{X}, \mathcal{Y}}$.
  



    We will now show that every solution graph $\mathcal{S}$ defines a binary decision tree, and define this correspondence explicitly.
    For every OR Node $o_{\mathcal{I},d} \in \mathcal{S}$, we create a corresponding node $n_o \in T$.
    Since the root node $r = o_{[n], 0} \in \mathcal{S}$, we must have that $T$ is nonempty.
    Furthermore, by the definition of solution graph, we must have that for every OR node $o_{\mathcal{I},d} \in \mathcal{S}$,  either $a_{\mathcal{I},d, f} \in \mathcal{S}$ for some value of $f$ or its corresponding terminal node $t_{\mathcal{I},d} \in \mathcal{S}$.
    If $a_{\mathcal{I},d, f} \in \mathcal{S}$ for some value of $f$, then by the definition of solution graph over $\mathcal{G}_{\mathcal{X}, \mathcal{Y}}$, we must have $o^L \coloneqq o_{\mathcal{I}(n)|_{f(n) = 0}, d(n)+1} \in \mathcal{S}$ and $o^R \coloneqq o_{\mathcal{I}(n)|_{f(n) = 1}, d(n)+1} \in \mathcal{S}$. 
    In this case, we connect node $n_o$ to each of $n_{o^L}$ and $n_{o^R}$ in $T$ with a directed edge.
    (If $t_{\mathcal{I},d} \in \mathcal{S}$, then $o_{\mathcal{I},d} \in \mathcal{S}$ corresponds to a leaf node in $T$).

    We now show that this process gives rise to a binary decision tree, i.e., that $T$ is a binary decision tree.
    First, we note that by construction, any $n_o \in T$ that has an outgoing edge must have exactly two outgoing edges, say to $n_{o^L}$ and $n_{o^R}$.
    Furthermore, these edges exist if and only if the corresponding OR nodes in the solution graph $\mathcal{S}$ are connected through directed edges through an AND node $a_f$ for some $f$.
    In this case, the subsets of the data at $n_{o^L}$ and $n_{o^R}$ must correspond to the subset of the data at $n_o$ split by feature $f$.
    This implies that the subset of data that reaches $n_{o^L}$ and $n_{o^R}$ in $T$ must be a strict subset of the data that reaches $n_{o}$. 
    Furthermore, we note that every node in $T$ must be reachable from the root node of $T$, $n_r$ (which corresponds to the start node in $\mathcal{S}$), so $T$ is connected.
    Together, these observations imply that $T$ is a binary decision tree.

    Finally, we note that these two constructions are inverses.
    Any binary decision tree $T$ can used to induce a solution graph $\mathcal{S}$ over $\mathcal{G}_{\mathcal{X}, \mathcal{Y}}$, and $\mathcal{S}$ in turn induces a binary decision tree equivalent to $T$.
    This proves the claim.
\end{proof}

\begin{proof}[Proof of Theorem \ref{thm:andor_propto_tree_ll}] 

    Let $L_O$ and $I_O$ denote the sets of OR nodes corresponding with leaf nodes and internal nodes, respectively, in the tree represented by solution $\mathcal{S}$. Then the cost of a solution graph $\mathcal{S}$ over $\mathcal{G}_{\mathcal{X}, \mathcal{Y}}$ is 
    \begin{align*}
        & -\sum_{l \in L_O} \log p_\text{stop}(\mathcal{X}|_l, \mathcal{Y}|_l, d(l)) \\
        & \hspace{2.5em} -\sum_{m \in I_O} \log p_\text{node}(d(m))
    \end{align*}

    where for any node with depth $d$, we have 
    \begin{equation}
        p_\text{node} \coloneqq
        \begin{cases}
          p_\text{split}(d), & \text{if node is an internal node}   \\
          p_\text{stop}(d, \mathcal{X}_o)& \text{if node is a leaf node}
        \end{cases}
    \end{equation}

    by our choices of $p_\text{split}$ and $p_\text{stop}$ in Section \ref{sec:bdts_and_ao_graphs}.
    We can further simplify this to

    \begin{align*}
        -\sum_{L_O} \log P(\mathcal{Y}|_\text{leaf} | \mathcal{X}|_\text{leaf}, T) &-\sum_{I_O} \log P(T | \mathcal{X}) \\
        &= - \log P(\mathcal{Y}, T | \mathcal{X}) \\
    \end{align*}

by the definition of BDT.
    
\end{proof}

\begin{lemma}  
    \label{lem:perfect-leaf-split-lb}
    For the leaf likelihood given in Equation \ref{def:leaf-lik}, we have that $\ell_\text{leaf}(a, 0)\ell_\text{leaf}(0, b) \geq \ell_\text{leaf}(a,b)$ for integers $a, b \geq 0$.
\end{lemma}

\begin{proof}[Proof of \ref{lem:perfect-leaf-split-lb}]
    If either $a$ or $b$ is $0$, then the claim is trivially true as $\ell_\text{leaf}(0, 0) = 1$. Therefore, we assume both $a$ and $b$ are greater than $0$.
    Using the facts that $B(x, y) = \frac{\Gamma(x)\Gamma(y)}{\Gamma(x + y)}$, where $\Gamma$ is the gamma function, and $\Gamma(x+1) = x\Gamma(x)$ we have that:
    \begin{align*}
        & \ell_\text{leaf}(a, 0)\ell_\text{leaf}(0, b) \maybegeq \ell_\text{leaf}(a,b) \\
        \iff &\frac{B(a + \rho_0, \rho_1)B(\rho_0, b + \rho_1)}{B(\rho_0, \rho_1)B(\rho_0, \rho_1)} \maybegeq \frac{B(a + \rho_0, b + \rho_1)}{B(\rho_0, \rho_1)}\\
        \iff &\frac{\left(\frac{\Gamma(a+\rho_0)\Gamma(\rho_1)}{\Gamma(a+\rho_0+\rho_1)}\right)\left(\frac{\Gamma(\rho_0)\Gamma(b+\rho_1)}{\Gamma(b+\rho_0+\rho_1)}\right)}{\left(\frac{\Gamma(\rho_0)\Gamma(\rho_1)}{\Gamma(\rho_0+\rho_1)}\right)\left(\frac{\Gamma(\rho_0)\Gamma(\rho_1)}{\Gamma(\rho_0+\rho_1)}\right)} \maybegeq \frac{\frac{\Gamma(a+\rho_0)\Gamma(b+\rho_1)}{\Gamma(a+b+\rho_0+\rho_1)}}{\frac{\Gamma(\rho_0)\Gamma(\rho_1)}{\Gamma(\rho_0+\rho_1)}}\\
        \iff &\Gamma(\rho_0 + \rho_1)\Gamma(a + b + \rho_0 + \rho_1) \maybegeq \\
        & \hspace{0.5em}\Gamma(a + \rho_0 + \rho_1)\Gamma(b + \rho_0 + \rho_1)\\
        \iff & \left(\prod_{i=0}^{a+b-1}(i+\rho_0+\rho_1)\right) \Gamma(\rho_0 + \rho_1)^2 \maybegeq \\
        & \hspace{0.5em}\left(\prod_{i=0}^{a-1}(i+\rho_0+\rho_1)\right) \left(\prod_{i=0}^{b-1}(i+\rho_0+\rho_1)\right) \Gamma(\rho_0 + \rho_1)^2 \\
        \iff & \prod_{i=a}^{a+b-1}(i+\rho_0+\rho_1) \maybegeq \prod_{i=0}^{b-1}(i+\rho_0+\rho_1)
    \end{align*}

    Since $\rho_0, \rho_1 > 0$ and $a > 0$, we must have the the LHS $\geq$ RHS, which proves the claim.
\end{proof}

\begin{lemma}  
    \label{lem:perfect-leaf-no-further-split-lb}
    For the leaf likelihood given in Definition \ref{def:leaf-lik}, we have that $\ell_\text{leaf}(a + b, 0) \geq \ell_\text{leaf}(a,0)\ell_\text{leaf}(b, 0)$ and $\ell_\text{leaf}(0, a+b) \geq \ell_\text{leaf}(0,a)\ell_\text{leaf}(0, b)$ for integers $a, b \geq 0$.
\end{lemma}

\begin{proof}[Proof of \ref{lem:perfect-leaf-no-further-split-lb}]
    If either $a$ or $b$ is $0$, then the claim is trivially true as $\ell_\text{leaf}(0, 0) = 1$. Therefore, we assume both $a$ and $b$ are greater than $0$.
    Using the facts that $B(x, y) = \frac{\Gamma(x)\Gamma(y)}{\Gamma(x + y)}$ and $\Gamma(x+1) = x\Gamma(x)$ we have that:
    
    \begin{align*}
        & \ell_\text{leaf}(a + b, 0) \maybegeq \ell_\text{leaf}(a,0)\ell_\text{leaf}(b, 0) \\
        \iff &\frac{B(a + b + \rho_0, \rho_1)}{B(\rho_0, \rho_1)} \maybegeq \frac{B(a + \rho_0, \rho_1)B(b + \rho_0, \rho_1)}{B(\rho_0, \rho_1)B(\rho_0, \rho_1)} \\
        \iff &B(\rho_0, \rho_1) B(a+b+\rho_0, \rho_1) \maybegeq \\
        & \hspace{6em}B(a+\rho_0, \rho_1)B(b+\rho_0, \rho_1)\\
        \iff &\frac{\Gamma(\rho_0)\Gamma(\rho_1)}{\Gamma(\rho_0+\rho_1)} \frac{\Gamma(a+b+\rho_0)\Gamma(\rho_1)}{\Gamma(a+b+\rho_0+\rho_1)} \maybegeq \\
        & \hspace{2em} \frac{\Gamma(a+\rho_0)\Gamma(\rho_1)}{\Gamma(a+\rho_0+\rho_1)}\frac{\Gamma(b+\rho_0)\Gamma(\rho_1)}{\Gamma(b+\rho_0+\rho_1)} \\
        \iff &\Gamma(\rho_0)\Gamma(a+b+\rho_0)\Gamma(a+\rho_0+\rho_1)\Gamma(b+\rho_0+\rho_1) \maybegeq \\
        & \hspace{1em}\Gamma(\rho_0+\rho_1)\Gamma(a+b+\rho_0+\rho_1)\Gamma(a+\rho_0)\Gamma(b+\rho_0) \\ 
        \iff &\frac{\Gamma(\rho_0)\Gamma(a+b+\rho_0)}{\Gamma(a+\rho_0)\Gamma(b+\rho_0)} \maybegeq \\
        & \hspace{3em}\frac{\Gamma(\rho_0+\rho_1)\Gamma(a+b+\rho_0+\rho_1)}{\Gamma(a+\rho_0+\rho_1)\Gamma(b+\rho_0+\rho_1)} \\ 
        \end{align*}
        \begin{align*}
        \iff &\frac{\left(\prod_{i=0}^{a+b-1} (i + \rho_0)\right)}{(\prod_{i=0}^{b-1} (i + \rho_0))(\prod_{i=0}^{a-1} (i + \rho_0))} \maybegeq \\
        &\hspace{2em} \frac{\left(\prod_{i=0}^{a+b-1} (i + \rho_0+\rho_1)\right)}{\left(\prod_{i=0}^{b-1} (i + \rho_0+\rho_1)\right)\left(\prod_{i=0}^{a-1} (i + \rho_0 + \rho_1)\right)} \\
        \iff &\prod_{i=0}^{b-1}\left(\frac{a+i+\rho_0}{i+\rho_0}\right) \maybegeq \prod_{i=0}^{b-1}\left(\frac{a+i+\rho_0 + \rho_1}{i+\rho_0 + \rho_1}\right) \\
        \iff &\prod_{i=0}^{b-1}\left(1 + \frac{a}{i+\rho_0}\right) \maybegeq \prod_{i=0}^{b-1}\left(1+\frac{a}{i+\rho_0 + \rho_1}\right)
    \end{align*}

    Since $\rho_0, \rho_1 > 0$ and $a > 0$, each term in the product on the LHS is positive and greater than or equal to the corresponding term on the RHS, so we have that LHS $\geq$ RHS.
    This proves that $\ell_\text{leaf}(a + b, 0) \geq \ell_\text{leaf}(a,0)\ell_\text{leaf}(b, 0)$ for integers $a, b \geq 0$.
    The proof for $\ell_\text{leaf}(0, a+b) \geq \ell_\text{leaf}(0,a)\ell_\text{leaf}(0, b)$ for integers $a, b \geq 0$ follows from the symmetry of $\ell_\text{leaf}$ in its arguments.
\end{proof}


\begin{proof}[Proof of Theorem \ref{thm:consistency_of_h}]
    First, we show that $h$ is consistent across any AND node $a_{\mathcal{I}, d, f}$ with children $o_{\mathcal{I}|_{f=0}, d+1}, o_{\mathcal{I}|_{f=1}, d+1}$.
    This follows directly from the definition of $h$:
    \begin{align}
        h(a_{\mathcal{I}, d, f}) &= \\
        &\texttt{cost}(a_{\mathcal{I}, d, f}, o_{\mathcal{I}|_{f=0}, d+1})\\
        &+ h(o_{\mathcal{I}|_{f=0}, d+1}) \\
        &+ \texttt{cost}(a_{\mathcal{I}, d, f}, o_{\mathcal{I}|_{f=1}, d+1}) \\
        &+ h(o_{\mathcal{I}|_{f=1}, d+1})
    \end{align}
    Next, we show that $h$ is consistent for any OR node $o_{\mathcal{I}, d}$ with children $t_{\mathcal{I}, d}, a_{\mathcal{I}, d, 1}, \dots, a_{\mathcal{I}, d, F}$. \\
    Case 1: 
    \begin{align}
        \ell_{\text{leaf}}(c^1(\mathcal{I}), c^0(\mathcal{I})) \geq p_{\text{split}}(d) \ell_{\text{leaf}}(c^1(\mathcal{I}), 0) \ell_{\text{leaf}}(0, c^0(\mathcal{I}))
    \end{align}
    In this case, we see that the heuristic is consistent for the terminal node child of $o_{\mathcal{I}, d}$: $t_{\mathcal{I}, d}$:
    \begin{align}
       h(o_{\mathcal{I}, d}) &= -\log \ell_{\text{leaf}}(c^1(\mathcal{I}), c^0(\mathcal{I})) \\
        &= \texttt{cost}(o_{\mathcal{I}, d}, t_{\mathcal{I}, d})
   \end{align}
    We also have the following:
    \begin{align}
        h(o_{\mathcal{I}, d}) \leq \\
        &-\log p_{\text{split}}(d) \\
        &- \log \ell_{\text{leaf}}(c^1(\mathcal{I}), 0)\\
        &- \log \ell_{\text{leaf}}(0, c^0(\mathcal{I}))
    \end{align}
    It remains to show that the heuristic is consistent for all AND node children of $o_{\mathcal{I}, d}$: $a_{\mathcal{I}, d, 1}, \dots, a_{\mathcal{I}, d, F}$.
    
    Case 2:
    \begin{align}
        \ell_{\text{leaf}}(c^1(\mathcal{I}), c^0(\mathcal{I})) < p_{\text{split}}(d) \ell_{\text{leaf}}(c^1(\mathcal{I}), 0) \ell_{\text{leaf}}(0, c^0(\mathcal{I}))
    \end{align} 
    In this case, we see that the heuristic is again consistent for the terminal node child of $o_{\mathcal{I}, d}$, $t_{\mathcal{I}, d}$:
   \begin{align}
       h(o_{\mathcal{I}, d}) &= \\
       &-\log p_{\text{split}}(d) \\
       &- \log \ell_{\text{leaf}}(c^1(\mathcal{I}), 0) \\
       &- \log \ell_{\text{leaf}}(0, c^0(\mathcal{I})) \\
       &\leq -\log \ell_{\text{leaf}}(c^1(\mathcal{I}), c^0(\mathcal{I})) \\
       &= \texttt{cost}(o_{\mathcal{I}, d}, t_{\mathcal{I}, d})
   \end{align}

    For the AND node children, we begin with the following:
    \begin{align}
        h(o_{\mathcal{I}, d}) &\leq \\
        &-\log p_{\text{split}}(d) \\
        &- \log \ell_{\text{leaf}}(c^1(\mathcal{I}), 0) \\
        &- \log \ell_{\text{leaf}}(0, c^0(\mathcal{I}))
    \end{align}
    As in Case 1, it remains to show that the heuristic is consistent for all AND node children of $o_{\mathcal{I}, d}$: $a_{\mathcal{I}, d, 1}, \dots, a_{\mathcal{I}, d, F}$.

    We will now show, for both Cases 1 and 2, that from this inequality, it follows that the heuristic is consistent across all AND node children of $o_{\mathcal{I}, d}$: $a_{\mathcal{I}, d, 1}, \dots, a_{\mathcal{I}, d, F}$.

    Applying Lemmas \ref{lem:perfect-leaf-split-lb} and \ref{lem:perfect-leaf-no-further-split-lb} and the symmetry of $\ell_{\text{leaf}}$, we have:
    \begin{align}
        h(o_{\mathcal{I}, d}) &\leq \\
        &-\log p_{\text{split}}(d) \\
        &- \log \ell_{\text{leaf}}(c^1(\mathcal{I}), 0) \\
        &- \log \ell_{\text{leaf}}(0, c^0(\mathcal{I})) \\
        &\leq \\
        &-\log p_{\text{split}}(d) \\
        &- \log \ell_{\text{leaf}}(c^1(\mathcal{I}|_{f=0}), 0)\ell_{\text{leaf}}(c^1(\mathcal{I}|_{f=1}), 0) \\
        &- \log \ell_{\text{leaf}}(0, c^0(\mathcal{I}|_{f=0}))\ell_{\text{leaf}}(0, c^0(\mathcal{I}|_{f=1})) \\
        &= \\
        &-\log p_{\text{split}}(d) \\
        &- \log \ell_{\text{leaf}}(c^1(\mathcal{I}|_{f=0}), 0) \ell_{\text{leaf}}(0, c^0(\mathcal{I}|_{f=0})) \\
        &- \log\ell_{\text{leaf}}(c^1(\mathcal{I}|_{f=1}), 0)\ell_{\text{leaf}}(0, c^0(\mathcal{I}|_{f=1})) \\
        &\leq \\
        &-\log p_{\text{split}}(d) \\
        &- \log \ell_{\text{leaf}}(c^1(\mathcal{I}|_{f=0}), c^0(\mathcal{I}|_{f=0})) \\
        &- \log \ell_{\text{leaf}}(c^1(\mathcal{I}|_{f=1}), c^0(\mathcal{I}|_{f=1}))
   \end{align}

   Applying Lemma \ref{lem:perfect-leaf-no-further-split-lb}, we have:
   \begin{align}
       h(o_{\mathcal{I}, d}) &\leq \\
       &-\log p_{\text{split}}(d) \\
       &- \log \ell_{\text{leaf}}(c^1(\mathcal{I}), 0) \\
       &- \log \ell_{\text{leaf}}(0, c^0(\mathcal{I})) \\
       &\leq \\
       &-\log p_{\text{split}}(d) \\
       &- \log \ell_{\text{leaf}}(c^1(\mathcal{I}|_{f=0}), 0) \ell_{\text{leaf}}(c^1(\mathcal{I}|_{f=1}), 0) \\
       &- \log \ell_{\text{leaf}}(0, c^0(\mathcal{I}|_{f=0})) \ell_{\text{leaf}}(0, c^0(\mathcal{I}|_{f=1})) \\
       &\leq \\
       &-\log p_{\text{split}}(d) \\
       &- 2\log p_{\text{split}}(d+1) \\
       &- \log \ell_{\text{leaf}}(c^1(\mathcal{I}|_{f=0}), 0) \\
       &- \log \ell_{\text{leaf}}(0, c^0(\mathcal{I}|_{f=0})) \\
       &- \log \ell_{\text{leaf}}(c^1(\mathcal{I}|_{f=1}), 0) \\
       &- \log \ell_{\text{leaf}}(0, c^0(\mathcal{I}|_{f=1}))
   \end{align}

   From the above two inequalities, we have that:
   \begin{align}
       h(o_{\mathcal{I}, d}) &\leq \texttt{cost}(o_{\mathcal{I}, d}, a_{\mathcal{I}, d, f}) + h(a_{\mathcal{I}, d, 1})
   \end{align}
   as required.

\end{proof}

\begin{corollary}[Admissibility of Perfect Split Heuristic]
\label{cor:admissibility_of_h}
The Perfect Split Heuristic $h$ defined in Definition \ref{def:heuristic} is admissible, i.e., given the true value of an OR node $f(o)$, we have that $h(o) \leq f(o)$, and given the true value of an AND node $f(a)$, we have that $h(a) \leq f(a)$.
\end{corollary}

\begin{lemma}
    \label{lem:correctness_of_ub}
    Across all iterations of \algnamenospace, $UB[o]$ represents the minimal cost of any partial solution rooted at OR node $o$ in $\mathcal{G}'$.
\end{lemma}

\begin{proof}[Proof of Lemma \ref{lem:correctness_of_ub}]
    We prove this via induction on iteration.
    After the first iteration, there is only terminal node in $\mathcal{G}'$, and only one valid solution $\mathcal{S}$ exists in $\mathcal{G}'$: $\mathcal{S}_0 = \{ o_{[N], 0}, t_{[N], 0} \}$.
    Thus, no nodes other than $r = o_{[N], 0}$ have valid partial solutions.
    At this point, $UB[r] = \texttt{cost}(o_{[N], 0}, t_{[N], 0}) = \texttt{cost}(\mathcal{S})$ and $UB$ is undefined on all other nodes, as required.

    In each future iterations, there is at most one terminal node $t^*$ added to $\mathcal{G}'$.
    We will show that for any OR node $o_{\mathcal{I}, d}$, $UB[o_{\mathcal{I}, d}]$ represents the minimal cost of any partial solution rooted at $o_{\mathcal{I}, d}$.
    We prove this over induction over the size of $\mathcal{I}$ of $o_{\mathcal{I}, d} \in \mathcal{G}'$. \\
    
    When $|\mathcal{I}| = 1$, any split will lead to an empty subtree, meaning if $t_{\mathcal{I}, d} \in \mathcal{G}'$ then $\{o_{\mathcal{I}, d}, t_{\mathcal{I}, d}\}$ is the only partial solution rooted at $o_{\mathcal{I}, d}$ and otherwise no such partial solution exists.
    If $t_{\mathcal{I}, d} \not\in \mathcal{G}'$, this implies that $o_{\mathcal{I}, d} \not\in \mathcal{E}$, meaning $UB[o_{\mathcal{I}, d}]$ is undefined.
    If $t_{\mathcal{I}, d} \in \mathcal{G}'$ and $|\mathcal{I}| = 1$, then $UB[o_{\mathcal{I}, d}]$ is $\texttt{cost}(o_{\mathcal{I}, d}, t_{\mathcal{I}, d})$, as required. \\
    
    When $|\mathcal{I}| > 1$, we have two cases: \\
    
    Case 1: There exists a minimal cost partial solution rooted at $o_{\mathcal{I}, d}$ which does not contain $t^*$. \\
        In this case, the cost of this partial solution is still the minimal cost across any partial solution rooted at $o_{\mathcal{I}, d}$.
        $UB[o_{\mathcal{I}, d}]$ remains unchanged in this case.
        (Otherwise a child's $UB$ must have been updated to a value such that there is now a partial solution rooted at $o_{\mathcal{I}, d}$ which contains that child and its new minimal cost partial solution.
        However, since the only terminal node added to $\mathcal{G}'$ this iteration was $t^*$, this implies that the child's new minimal cost solution must contain $t^*$, which is a contradiction.)
        Since $UB[o_{\mathcal{I}, d}]$ was not changed, our inductive hypothesis over iterations states that $UB[o_{\mathcal{I}, d}]$ still represents the minimal cost of any partial solution rooted at OR node $o$ in $\mathcal{G}'$. \\
    
    Case 2: All minimal cost partial solutions rooted at $o_{\mathcal{I}, d}$ contain $t^*$. \\
        Consider a child $c^*$ that is part of some such minimal cost partial solution.
        In this case, our inductive hypothesis over $|\mathcal{I}|$ gives us that $UB$ is correctly updated in this iteration.
        It follows that $o_{\mathcal{I}, d}$ will be added to the queue in $\texttt{updateUpperBounds}$ after this child because it must have lower depth.
        As a result, $UB[o_{\mathcal{I}, d}]$ will be set to the minimal cost of any partial solution rooted at $o_{\mathcal{I}, d}$, as required. \\
    
    We conclude that, in either case, $UB[o]$ represents the minimal cost of any partial solution rooted at OR node $o$ in $\mathcal{G}'$.
\end{proof}

\begin{lemma}
    \label{lem:correctness_of_get_solution}
    \texttt{getSolution}($o_{\mathcal{I}, d}$) outputs a minimal cost partial solution of $\mathcal{G}'$ rooted at OR node $o_{\mathcal{I}, d}$.
\end{lemma}

\begin{proof}[Proof of Lemma \ref{lem:correctness_of_get_solution}]
    We will show that \texttt{getSolution}($o_{\mathcal{I}, d}$) outputs a minimal cost partial solution of $\mathcal{G}'$ rooted at OR node $o_{\mathcal{I}, d}$ via induction on $|\mathcal{I}|$.
    When $|\mathcal{I}| = 1$, any splits will lead to a empty subtrees, so $\texttt{getSolution}(o_{\mathcal{I}, d})$ must return $\{o_{\mathcal{I}, d}, t_{\mathcal{I}, d}\}$.
    When $|\mathcal{I}| > 1$, \texttt{getSolution}($o_{\mathcal{I}, d}$) will either stop for a minimal cost or recurse on a split that yields minimal $UB$ value.
    Lemma \ref{lem:correctness_of_ub} shows that the $UB$ values of $o_{\mathcal{I}, d}$ and its children are equal to the minimal cost across all partial solutions rooted at each of these respective nodes.
    As a result, if $\texttt{getSolution}$ stops, then $\{o_{\mathcal{I}, d}, t_{\mathcal{I}, d}\}$ is a minimal cost partial solution.
    Otherwise, if $\texttt{getSolution}$ splits on feature $f$, $\{o_{\mathcal{I}, d}, a_{\mathcal{I}, d, f}\} \, \cup \texttt{getSolution}(o_{\mathcal{I}|_{f=0}, d + 1}) \, \cup \texttt{getSolution}(o_{\mathcal{I}|_{f=0}, d + 1})$ is also a minimal cost partial solution by the inductive hypothesis.
\end{proof}

\begin{proof}[Proof of Theorem \ref{thm:anytime}]
    We will show that upon early termination, \algname always returns a minimal cost solution within the explicit subgraph $\mathcal{G}' \subset G$ explored by \algnamenospace.
    From Lemma \ref{lem:correctness_of_get_solution}, we have that a minimal cost solution of $\mathcal{G}'$ is output by \algnamenospace, even upon early termination.
\end{proof}

\begin{lemma}
   \label{lem:correct_lower_bound}
   The lower bounds $LB$ represent correct lower bounds on the true value of a node in every iteration. For any OR node $o$ with true value $f(o)$, we have that $LB[o] \leq f(o)$ and for any AND node $a$ with true value $f(a)$, we have that $LB[a] \leq f(a)$.
\end{lemma}

\begin{proof}[Proof of Lemma \ref{lem:correct_lower_bound}]
    We will show that for any OR node $o_{\mathcal{I}, d}$, we have that $LB[o_{\mathcal{I}, d}] \leq f(o_{\mathcal{I}, d})$.
    This follows from Corollary $\ref{cor:admissibility_of_h}$.
    Throughout \algnamenospace, $LB[o_{\mathcal{I}, d}]$ is set to either:
    \begin{enumerate}
        \item $h(o_{\mathcal{I}, d})$
        \item $\min_{c \in \{t_{\mathcal{I}, d}, a_{\mathcal{I}, d, 1}, \dots, a_{\mathcal{I}, d, F}\}} \texttt{cost}(o_{\mathcal{I}, d}, c) + LB[c]$
    \end{enumerate}
    In the first case, Corollary \ref{cor:admissibility_of_h} gives us that $h(o_{\mathcal{I}, d}) \leq f(o_{\mathcal{I}, d})$.
    In the second case, we induct on iteration. First, though \algname does not query $LB$ for nodes on which a value has not yet been assigned, we will assume for the purpose of this proof that $LB$ defaults to $0$. Thus, before the first iteration, $LB$ is $0$ across all nodes. Our cost function $\texttt{cost}$ is nonnegative, so $LB[o_{\mathcal{I}, d}] = 0 \leq f(o_{\mathcal{I}, d})$ must hold. For future iterations then, we have the following for any node $o_{\mathcal{I}, d}$ on which $LB$ is defined:
    \begin{align}
        LB[o_{\mathcal{I}, d}] &:= \min_{c \in \{t_{\mathcal{I}, d}, a_{\mathcal{I}, d, 1}, \dots, a_{\mathcal{I}, d, F}\}} \texttt{cost}(o_{\mathcal{I}, d}, c) + LB[c] \\
        &\leq \min_{c \in \{t_{\mathcal{I}, d}, a_{\mathcal{I}, d, 1}, \dots, a_{\mathcal{I}, d, F}\}} \texttt{cost}(o_{\mathcal{I}, d}, c) + f(c) \\
        &\leq f(o_{\mathcal{I}, d})
    \end{align}
    Thus, $LB$ is a true lower bound, as required.
\end{proof}


\begin{lemma}
    \label{lem:gap_lb_ub}
    For any OR node $o$, if $LB[o] < UB[o]$, then $o$ must have some OR node descendant $o'$ such that $o' \not\in \mathcal{E}$.
\end{lemma}

\begin{proof}[Proof of Lemma \ref{lem:gap_lb_ub}]
    We prove the contrapositive. 
    Consider any OR node $o_{\mathcal{I}, d}$. We show that if every OR node descendant of $o_{\mathcal{I}, d}$ is in $\mathcal{E}$, then $LB[o_{\mathcal{I}, d}] = UB[o_{\mathcal{I}, d}]$. 
    We show this via induction on $|\mathcal{I}|$.
    When $|\mathcal{I}| = 1$, splitting further incurs infinite cost, so $LB[o_{\mathcal{I}, d}] = UB[o_{\mathcal{I}, d}] = \texttt{cost}(o_{\mathcal{I}, d}, t_{\mathcal{I}, d})$.
    When $|\mathcal{I}| > 1$, if every OR node descendant of $o_{\mathcal{I}, d}$ is in $\mathcal{E}$, then $o_{\mathcal{I}, d} \in \mathcal{E}$ must also hold. Since the OR node descendants of the children of $o_{\mathcal{I}, d}$ must be in $\mathcal{E}$ as well, they must all have matching $UB$ and $LB$ by inductive hypothesis, meaning $LB[o_{\mathcal{I}, d}] = UB[o_{\mathcal{I}, d}]$.
\end{proof}

\begin{lemma}
    \label{lem:expands_node}
    If $LB[r] < UB[r]$, then \texttt{findNodeToExpand} returns an unexpanded OR node $o \not\in \mathcal{E}$.
\end{lemma}

\begin{proof}
    This follows from Lemma \ref{lem:gap_lb_ub} and that \texttt{findNodeToExpand} selects an AND node with the lowest lower bound and the child of this AND node with the largest gap between $LB$ and $UB$, meaning a nonzero gap is chosen when one exists.
\end{proof}

\begin{proof}[Proof of Theorem \ref{thm:finiteness}]
    Lemma $\ref{lem:gap_lb_ub}$ gives us that $LB[r] = UB[r]$ must hold upon exhaustive exploration of the search space. From Lemma \ref{lem:expands_node}, we have that \algname will always expand a new node every iteration that \algname has not completed. Since $\mathcal{G}$ is finite, as discussed after Definition \ref{def:bcart-andor-graph}, it follows that \algname must eventually complete.
\end{proof}




    

\begin{proof}[Proof of Theorem \ref{thm:correctness}]
    This follows directly from Theorem \ref{thm:finiteness}, Lemma \ref{lem:correctness_of_ub}, and Lemma \ref{lem:correct_lower_bound}.
\end{proof}



\section{Experiment Details and Additional Experiments}
\label{app:add_experiments}

\subsection{Experiment Details}
\label{app:exp_details}

In Section \ref{sec:exps}, we compared the performance of \algname against various state-of-the-art baselines.
In this subsection, we describe those baselines and the experiments in more detail.

\subsubsection{Speed Comparisons against MCMC and SMC}
In this set of experiments in Section \ref{sec:exps}, we compared against the Sequential Monte Carlo (SMC) and Markov-Chain Monte Carlo (MCMC) methods \cite{smc} which sample from the BCART posterior from \citet{chipmanBayesianCARTModel1998}.
We used the posterior distribution hyperparameters for each algorithm specified in \citet{smc}. To gather results for each baseline with varying times, we set the number of islands in SMC to 10, 30, 100, 300, and 1000 and the number of iterations for MCMC to 10, 30, 100, 300, and 1000. For each of these settings we ran 10 iterations with different random seeds and recorded the average time across iterations, the mean log posterior of the tree with the highest log posterior discovered in each run, and a 95\% bootstrapped confidence interval of the average of highest log posteriors discovered across all runs. \algname was run with number of expansions limited to 10, 30, 100, 300, and 1000, 3000, 10000, 30000, and 100000. We also ran one additional run of MAPTree with a 10 minute time limit. Since \algname is a deterministic algorithm, we did a single run for each of these settings, recording the time and log posterior of the returned tree.
Runs which ran out of memory on our computing cluster with 16 GB of RAM in the time limit were discarded.

Note that, given that \algnamenospace, the SMC baseline, and MCMC baseline all explore the same posterior, measuring their relative generalization performance would not be meaningful; it is more meaningful to measure how quickly they can explore the posterior $P(T|\mathcal{X}, \mathcal{Y})$ to discover the maximum a posteriori tree.
Given infinite runtime, all three algorithms (\algnamenospace, SMC, and MCMC) should recover the maximum a posteriori tree from the BCART posterior.
However, recent work has proven that algorithms such as SMC and MCMC experience long mixing times on the BCART posterior \cite{kimMixingRatesBayesian2023}. 
Our experiments are consistent with these observations; we find that \algname is able to find higher likelihood trees faster than the SMC and MCMC baselines in most datasets.
Furthermore, in 5 of the 16 datasets, \algname is able to recover a maximum a posteriori tree and provide a certificate of optimality.

\subsubsection{Fitting a Synthetic Dataset}
\label{sec:synth-app}
In the remaining two sets of experiments, we compare \algnamenospace's generalization performance and model size to baseline state-of-the-art ODT and OSDT algorithms.
ODT algorithms search for trees which minimize misclassification error given a maximum depth.
OSDT algorithms search for trees which minimize the same objective but with an added per-leaf sparsity penalty in lieu of a hard depth constraint.
We use DL8.5 \cite{aglinLearningOptimalDecision2020} as our baseline ODT algorithm and GOSDT \cite{linGeneralizedScalableOptimal2020} as our baseline OSDT algorithm.
Note that these algorithms maximize different objectives than \algname and do not explicitly explore the posterior $P(T|\mathcal{X}, \mathcal{Y})$, so we are primarily interested in the different methods' generalization performance.
We also use CART with constrained depth \cite{breimanClassificationRegressionTrees1984}, as the baseline representative of greedy, top-down approaches.
All of these baselines are sensitive to their choices of hyperparameters, in particular the maximum depth for CART and DL8.5, and sparsity penalty for GOSDT.
We experimented with these hyperparameters to find the best-performing ones for each baseline algorithms, and presented the results for representative settings in Section \ref{sec:exps}.
In particular, we chose maximum depth $4$ for CART, maximum depth $4-5$ for DL8.5, and sparsity penalties $\frac{10}{32}$ and $\frac{1}{32}$ for GOSDT.
These two sparsity penalties were taken from \cite{linGeneralizedScalableOptimal2020}: the former was used in evaluation of GOSDT's speed and the latter was used in evaluation of its accuracy.
Our experiments set a time limit of 1 minute across all algorithms; the best tree discovered by each algorithm within this time limit was recorded.

The synthetic dataset was created via the process described in Section \ref{sec:exps}.


\subsubsection{Accuracy, Likelihood and Size Comparisons on Real World Benchmarks}

In this set of experiments, we compared \algname to the baseline algorithms as described in Appendix \ref{sec:synth-app} on the CP4IM datasets, and the hyperparameters for all algorithms were set to the same values.
Again, our experiments set a time limit of 1 minute across all algorithms; the best tree discovered by each algorithm within this time limit was recorded.

The metrics we measure are the per-sample test log likelihood and test accuracy (relative to the performance of CART), and the total number of nodes in the trained tree.
We performed stratified 10-fold cross validation on each dataset and recorded the average value across folds for each metric.
The average per-sample test log likelihood and test accuracy of CART with maximum depth 4 was subtracted from all baselines on each dataset to get the relative per-sample test log likelihood and test accuracy.
The plots in Figure \ref{fig:ac_ll_size} are box-and-whisker plots of the metric values across all 16 datasets, where each box represents the $25$th to $75$th percentile, whiskers extend out to at most $1.5\times$ the size of the box body, and the remaining points are marked as outliers.

\subsection{Additional Experiments}

In this subsection, we present additional experimental results that were omitted from the main paper due to space constraints.

Figure \ref{fig:synthetic} shows that \algname generates trees which out-perform both the greedy, top-down approaches and ODT methods in test accuracy for various training dataset sizes
and values of label corruption proportion $\epsilon$.

\begin{figure}
    \centering
    \includegraphics[width=\linewidth]{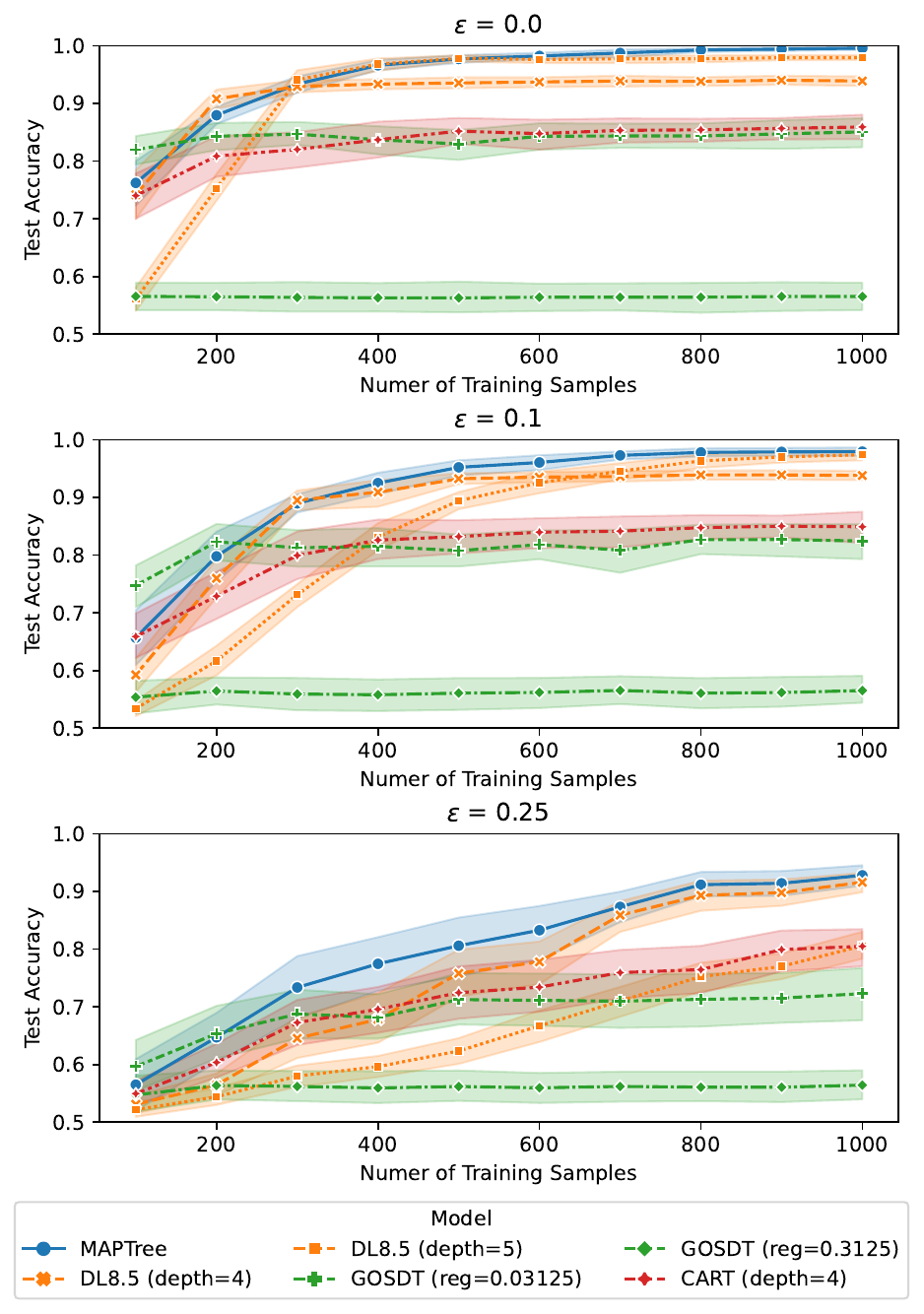}
    \caption{Test accuracy of \algname and various baseline algorithms as a function of training dataset size on the synthetic dataset, for different values of noise, $\epsilon$. \algname generates trees which outperform both the greedy, top-down approaches and ODT methods in test accuracy for various training dataset sizes and values of label corruption proportion $\epsilon$. 95\% confidence intervals are obtained via bootstrapping the results of 20 synthetic datasets generated independently at random.}
    \label{fig:synthetic}
\end{figure}

\subsubsection{Speed Comparisons against MCMC and SMC (Full)}

In this subsection, we include the results of the speed comparisons of \algname with SMC and MCMC on the remaining 12 datasets of the CP4IM dataset that were not presented in Section \ref{sec:exps} due to space constraints.
Figure \ref{fig:speed-full} demonstrates a similar trend on the additional datasets as was demonstrated by Figure \ref{fig:speed}, namely that \algname generally outperforms SMC and MCMC and is able to find trees with higher log posterior faster than the baseline algorithms.

\begin{figure*}
    \centering
    \includegraphics[width=\linewidth]{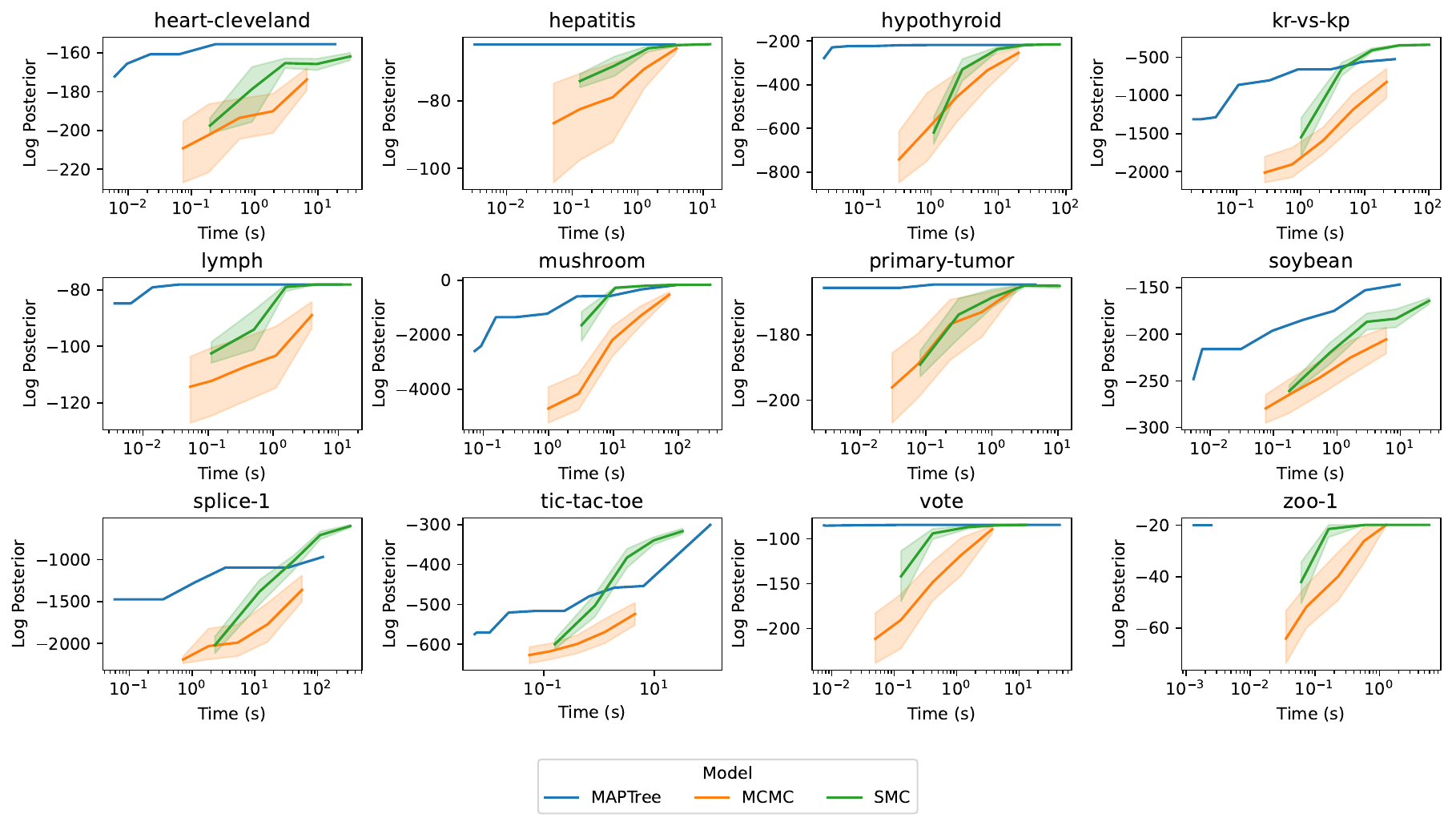}
    \caption{Comparison of \algnamenospace, SMC, and MCMC on 12 datasets. Curves are created by modifying the hyperparameters for each algorithm and measuring training time and log posterior of the data under the tree. Higher and further left is better, i.e., better log posteriors in less time. In 12 of the 16 datasets, \algname outperforms SMC and MCMC and is able to find trees with higher log posterior faster than the baseline algorithms. Furthermore, in 5 of the 16 datasets, \algname converges to the provably optimal tree, i.e., the maximum a posteriori tree. 95\% confidence intervals are obtained by bootstrapping the results of 10 random seeds and time is averaged across the 10 seeds.}
    \label{fig:speed-full}
\end{figure*}

\subsubsection{Fitting a Synthetic Dataset (Full)}

In this subsection, we include the results of the synthetic data experiment against benchmarks with a more exhaustive list of hyperparameters, which we omitted in Section \ref{sec:exps} due to space constraints.
Figure \ref{fig:synth-dl85}, \ref{fig:synth-gosdt}, and \ref{fig:synth-cart} demonstrate a similar trends as in Figure \ref{fig:synthetic}, namely that \algname generally the baseline algorithms with less training data, and is more robust to label noise than the baselines.

\begin{figure}
    \centering
    \includegraphics[width=\linewidth]{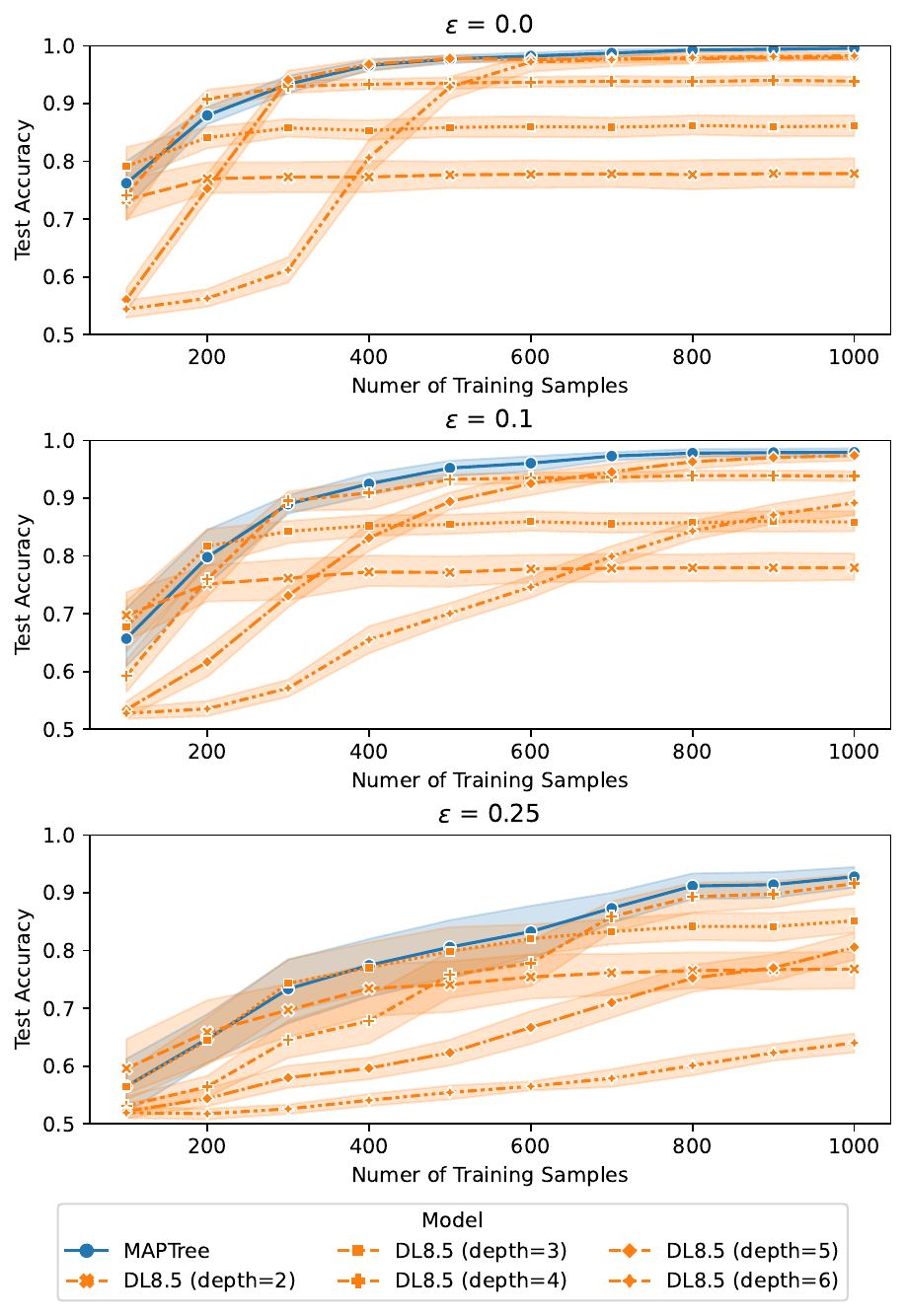}
    \caption{Test accuracy of \algname and DL8.5, for various hyperparameter settings of DL8.5, as a function of training dataset size on the synthetic dataset, for different values of noise, $\epsilon$. \algname generates trees which outperform DL8.5 for various training dataset sizes and values of label corruption proportion $\epsilon$. 95\% confidence intervals are derived by bootstrapping the results across 20 synthetic datasets generated independently at random.}
    \label{fig:synth-dl85}
\end{figure}

\begin{figure}
    \centering
    \includegraphics[width=\linewidth]{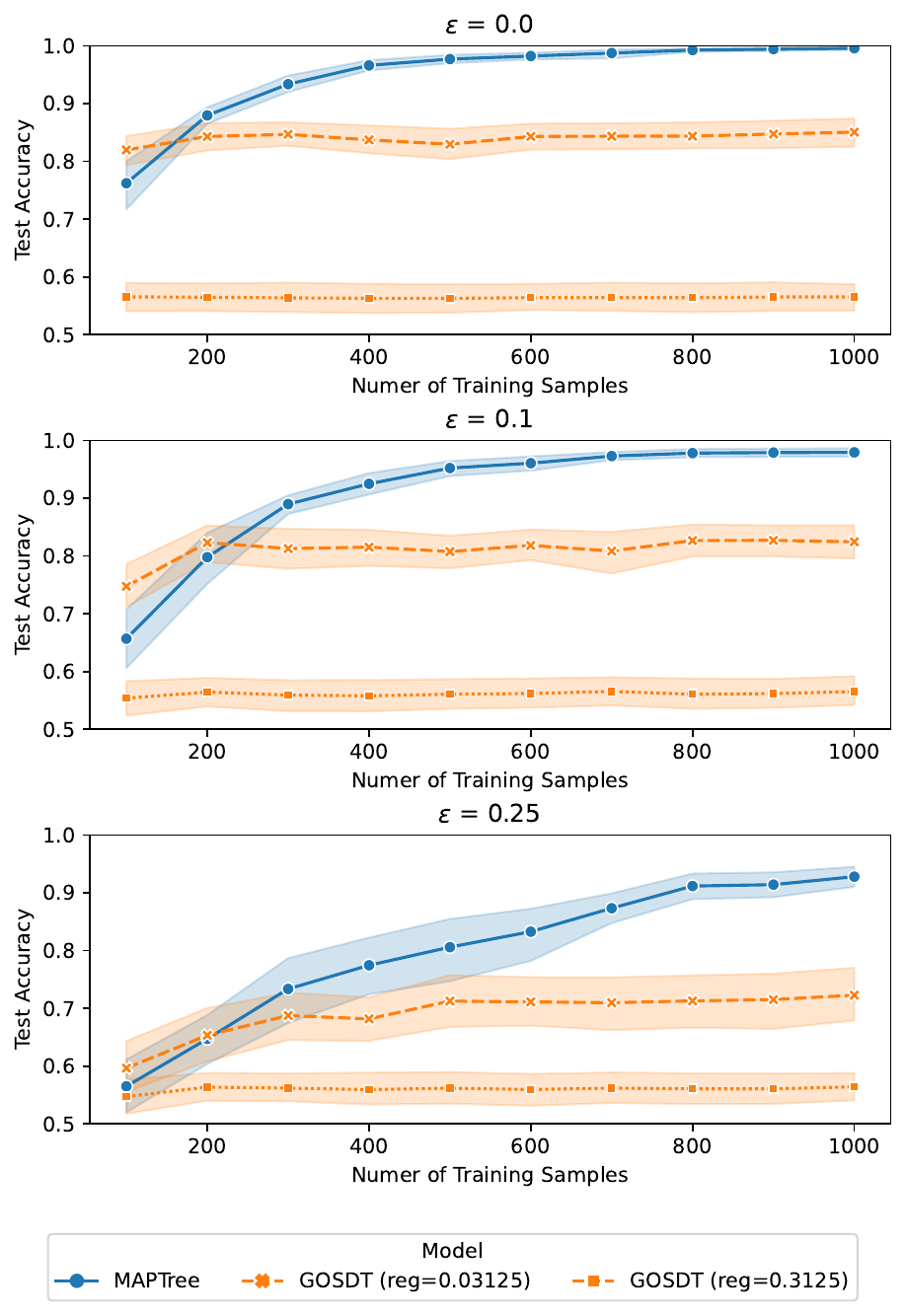}
    \caption{Test accuracy of \algname and GOSDT, for various hyperparameter settings of GOSDT, as a function of training dataset size on the synthetic dataset, for different values of noise, $\epsilon$. \algname generates trees which outperform GOSDT for various training dataset sizes and values of label corruption proportion $\epsilon$. 95\% confidence intervals are derived by bootstrapping the results across 20 synthetic datasets generated independently at random.}
    \label{fig:synth-gosdt}
\end{figure}

\begin{figure}
    \centering
    \includegraphics[width=\linewidth]{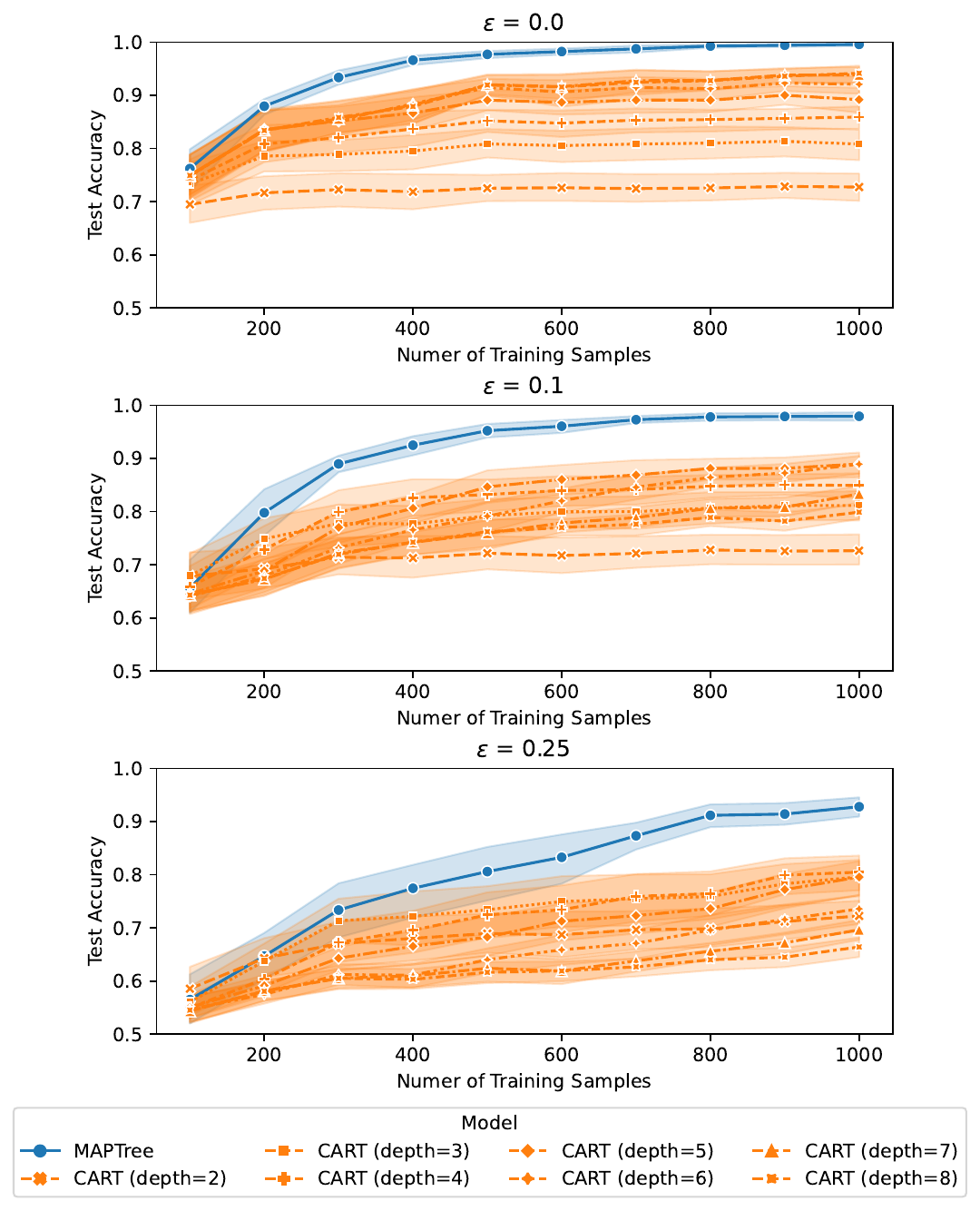}
    \caption{Test accuracy of \algname and CART, for various hyperparameter settings of CART, as a function of training dataset size on the synthetic dataset, for different values of noise, $\epsilon$. \algname generates trees which outperform CART for various training dataset sizes and values of label corruption proportion $\epsilon$. 95\% confidence intervals are derived by bootstrapping the results across 20 synthetic datasets generated independently at random.}
    \label{fig:synth-cart}
\end{figure}

\subsubsection{Accuracy, Likelihood and Size Comparison on Real World Benchmarks (Full)}

In this subsection, we include the results of accuracy, likelihood, and size comparisons against benchmarks with a more exhaustive list hyperparameter settings, which we omitted in Section \ref{sec:exps} due to space constraints.
Figure \ref{fig:ac_ll_size_full} demonstrates a similar trend as in Figure \ref{fig:ac_ll_size}, namely that \algname generally either a) outperforms the baseline algorithms in generalization performance, or b) performs comparably but with smaller trees. Further, we observe that CART and DL8.5 are sensitive to their hyperparameter settings: at higher maximum depths, both algorithms output much larger trees that do not perform any better than their shallower counterparts.

\begin{figure*}
    \centering
    \includegraphics[width=\linewidth]{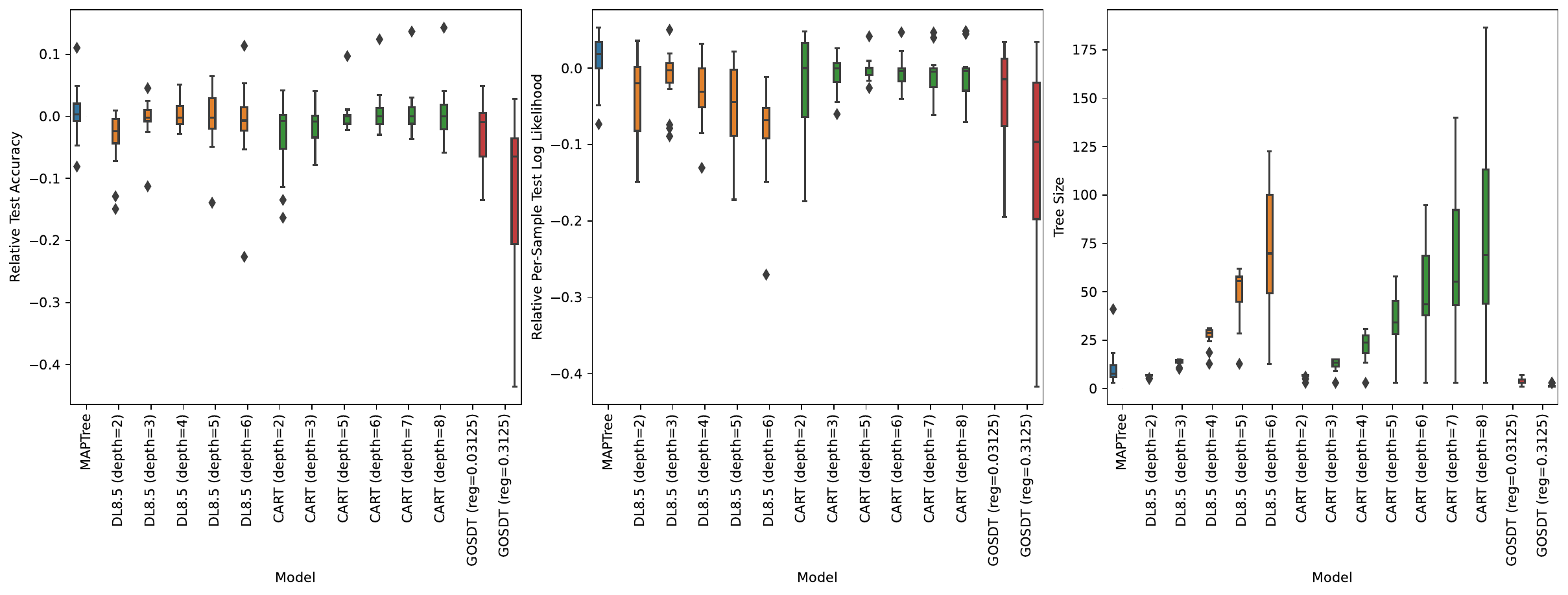}
    \caption{We run \algname and benchmarks with a more exhaustive list of hyperparameter settings on the 16 real world datasets from the CP4IM dataset \cite{gunsItemsetMiningConstraint2011}. Higher is better for the left and center subplots, and lower is better for the right subplot.}
    \label{fig:ac_ll_size_full}
\end{figure*}

\subsubsection{Hyperparameters of \algname}

In this subsection, We demonstrate that \algname is not sensitive to the choice of hyperparameters $\alpha$ and $\beta$.
We run \algname on all 16 benchmark datasets from CP4IM \cite{gunsItemsetMiningConstraint2011} with seven different hyperparameter settings of the prior distribution used in \algnamenospace:

\begin{enumerate}
    \item[0.] $\alpha = 0.999, \beta = 0.1$
    \item[1.] $\alpha = 0.99, \beta = 0.2$
    \item[2.] $\alpha = 0.95, \beta = 0.5$
    \item[3.] $\alpha = 0.9, \beta = 1.0$
    \item[4.] $\alpha = 0.8, \beta = 2.0$
    \item[5.] $\alpha = 0.5, \beta = 4.0$
    \item[6.] $\alpha = 0.2, \beta = 8.0$
\end{enumerate}

These prior specifications were chosen to cover a range of the hyperparameters $\alpha$ and $\beta$ that induce \algname to have different priors over the probability of splitting.
We expect the size of trees generated to be higher for earlier priors and lower for later priors. 
Relative to the first prior specification, the last prior specification assumes over $1000\times$ lower probability of splitting at depth $1$ a priori.

We measure the test accuracy relative to CART, the per-sample test log likelihood relative to CART, and the model size of the trees generated by \algname with the different hyperparameter settings.
Figure \ref{fig:hyperparams} demonstrates our results. 
\algname does not show significant sensitivity to its hyperparameters across any metric.

\begin{figure}
    \centering
    \includegraphics[width=\linewidth]{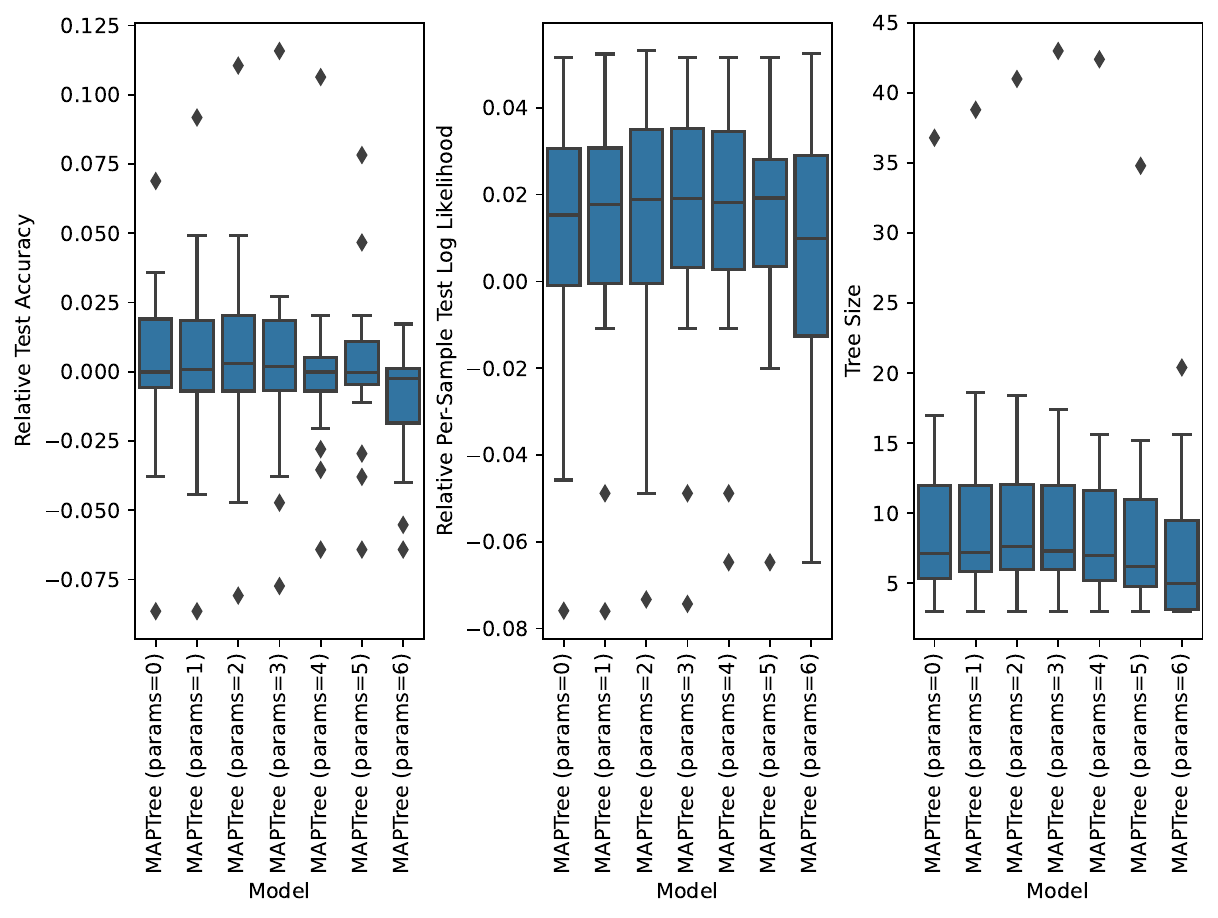}
    \caption{We run \algname on the 16 real world datasets from CP4IM \cite{gunsItemsetMiningConstraint2011} for various hyperparameter settings. We find that the trees generated by \algname are not sensitive to the exact values of hyperparameters.}
    \label{fig:hyperparams}
\end{figure}

\section{Implementation Details}
\label{app:implementation}

\subsection{Reversible Sparse Bitset}
In order to efficiently explore the search space $\mathcal{G}_{\mathcal{X}, \mathcal{Y}}$, described in Definition \ref{def:bcart-andor-graph}, \algname must be able to compactly represent subproblems and move between them efficiently.
Subproblems in $\mathcal{G}_{\mathcal{X}, \mathcal{Y}}$ correspond with a subset $\mathcal{I}$ of samples from the dataset $\mathcal{X}, \mathcal{Y}$ as well as a given depth.
We represent these subproblems using a reversible sparse bitset, a data structure also used in DL8.5 \cite{aglinLearningOptimalDecision2020}.
Reversible sparse bitsets represent $\mathcal{I}$ as a list of indexed bitstring blocks that correspond with nonempty subarrays of the current bitset.
These blocks also record a history of their previous values, allowing us to efficiently ``reverse'' a bitset to its previous value before branching into another region of the search space.
For more details on reversible sparse bitsets, we direct the reader to \citet{ijcai2018p192}.

\subsection{Caching Subproblems}
It is also necessary to identify equivalent subproblems in our search.
To do this, we cache each explored subproblem $o_{\mathcal{I}, d}$.
Previous ODT algorithms cache $(\mathcal{I}, d)$ explicitly or the path of splits used to reach the subproblem: loosely, $\text{path}(o_{\mathcal{I}, d})$ \cite{demirovicMurTreeOptimalDecision2022, aglinLearningOptimalDecision2020}.
The former approach takes up $O(N)$ memory per subproblem whereas the latter uses less memory but does not identify all equivalences between subproblems (i.e. two paths may result in the same subset of datapoints), meaning it is slower.
In \algnamenospace, we provide a probably correct cache which takes $O(1)$ memory per subproblem and always identifies equivalent subproblems.
The cache stores the depth $d$ and constructs a 128-bit hash of $\mathcal{I}$ for each subproblem $o_{\mathcal{I}, d}$.
We describe the 128-bit hashing function in Algorithm \ref{alg:hash}.

\begin{algorithm}
\caption{\texttt{subsetHash}}
\label{alg:hash}
\textbf{Input}: Subset of sample indices $\mathcal{I}$ \\
\textbf{Output}: 128-bit hash value $h$

\begin{algorithmic}[1] 
\STATE Let $\texttt{bitset}$ be a bitset of size $N$
\FORALL{$i \in [N]$}
    \STATE $\texttt{bitset}[i] := i \in \mathcal{I}$
\ENDFOR

\STATE Let $h_1, h_2$ be 64-bit integers.
\FORALL{$b \in [N/64]$}
    \STATE Let $\texttt{block}$ be $\texttt{bitset}[64b:64(b + 1)]$
    \STATE $h_1 := h_1 + \texttt{block} \times (377424577268497867)^b$
    \STATE $h_2 := h_2 + \texttt{block} \times (285989758769553131)^b$
\ENDFOR

\RETURN $(h_1, h_2)$
\end{algorithmic}
\end{algorithm}


\end{document}